\pdfoutput=1
\documentclass[ english, 
						colon, 
						dvipsnames,
						amsmath,
						amsthm,
						ejs,
						12pt]{imsart} %don't forget ejs & preprint for final version
\usepackage[T1]{fontenc}
\usepackage[latin9]{inputenc}
\usepackage{color}
\usepackage{babel}
\usepackage{bm}
\usepackage{amssymb}
\usepackage{graphicx}

\usepackage{geometry}
\usepackage[numbers,sort&compress]{natbib}

\makeatletter
\theoremstyle{plain}
\newtheorem{thm}{\protect\theoremname}
  \theoremstyle{plain}
  \newtheorem{lem}{\protect\lemmaname}
  \theoremstyle{plain}
  \newtheorem{prop}{\protect\propositionname}

\RequirePackage{hyperref}
\hypersetup{urlcolor=blue, citecolor=blue}
\usepackage{lmodern}
\usepackage[bb=boondox]{mathalfa}
\RequirePackage{graphicx}
\RequirePackage{xcolor}

\makeatother

  \providecommand{\lemmaname}{Lemma}
  \providecommand{\propositionname}{Proposition}
\providecommand{\theoremname}{Theorem}

\begin{document}
\global\long\def\mb#1{\hm{#1}}
\global\long\def\mbb#1{\mathbb{#1}}
\global\long\def\mc#1{\mathcal{#1}}
\global\long\def\mcc#1{\mathscr{#1}}
\global\long\def\mr#1{\mathrm{#1}}
\global\long\def\msf#1{\mathsf{#1}}
\global\long\def\mfk#1{\mathfrak{#1}}
\global\long\def\E{\mbb E}
\global\long\def\P{\mbb P}
\global\long\def\var{\ensuremath{\mr{var}}}
\global\long\def\T{\scalebox{0.55}{\ensuremath{\msf T}}}
\global\long\def\d{\ensuremath{\mr d}}
\global\long\def\tr{\ensuremath{\mr{tr}}}
\global\long\def\supp{\ensuremath{\,\mr{supp}}}
\global\long\def\sgn{\ensuremath{\,\mr{sgn}}}
\global\long\def\F{\scalebox{0.55}{\ensuremath{\mr F}}}
\global\long\def\argmax{\operatorname*{argmax}}
\global\long\def\argmin{\operatorname*{argmin}}
\global\long\def\defeq{\stackrel{\textup{\tiny def}}{=}}
\global\long\def\bbone{\mbb 1}
\global\long\def\norm#1{{\left\lVert #1\right\rVert }}

\makeatletter
\begin{frontmatter}
\title{Estimation from Nonlinear Observations via Convex Programming\\
with Application to Bilinear Regression}
\runtitle{Estimation from Nonlinear Observations via Convex Programming}

\begin{aug} 
%\thanksref{SPONSOR}
  \author{\fnms{Sohail} \snm{Bahmani}\ead[label=e1]{sohail.bahmani@ece.gatech.edu}}%,
  \address{School of Electrical and Computer Engineering,\newline Georgia Institute of Technology,\newline  Atlanta, GA 30332-0250\newline
  \printead{e1}
  }

  %\thankstext{t2}{Footnote to the first author with the `thankstext' command.}
  %\thankstext{SPONSOR}{The authors were supported in part by ....}
  \runauthor{S.\ Bahmani}

\end{aug}
\makeatother
%\maketitle

\begin{abstract}
We propose a computationally efficient estimator, formulated as a convex program, for a broad class of nonlinear regression problems
that involve \emph{difference of convex }(DC) nonlinearities. The proposed method can be viewed as a significant extension of the
``anchored regression'' method formulated and analyzed in \citep{Bahmani2017Anchored} for regression with convex nonlinearities. Our main assumption, in
addition to other mild statistical and computational assumptions, is availability of a certain approximation oracle for the average
of the gradients of the observation functions at a ground truth. Under this assumption and using a PAC-Bayesian analysis we show that the
proposed estimator produces an accurate estimate with high probability. As a concrete example, we study the proposed framework in the bilinear
regression problem with Gaussian factors and quantify a sufficient sample complexity for exact recovery. Furthermore, we describe a computationally tractable scheme that provably produces the required approximation oracle in the considered bilinear regression problem.
\end{abstract}

\begin{keyword}[class=MSC]
\kwd[Primary ]{62F10}
\kwd{90C25}
\kwd[; secondary ]{62P30}
\end{keyword}

\begin{keyword}
\kwd{nonlinear regression}
\kwd{convex programming}
\kwd{PAC-Bayesian analysis}
\kwd{bilinear regression}
\end{keyword}

\end{frontmatter}

\section{\label{sec:introduction}Introduction}

Let $f_{1}^{+},f_{2}^{+},\dotsc,f_{n}^{+}$ be i.i.d. copies of a
random \emph{convex} function $f^{+}:\mbb R^{d}\to\mbb R$. Similarly,
let $f_{1}^{-},f_{2}^{-},\dotsc,f_{n}^{-}$ be i.i.d. copies of a
random convex function $f^{-}:\mbb R^{d}\to\mbb R$. For simplicity,
we also assume that the functions $f^{+}$ and $f^{-}$ are differentiable.\footnote{Religiously, we may add ``almost everywhere almost surely.''}
We observe a parameter $\mb x_{\star}\in\mbb R^{d}$ indirectly through the measurements
\begin{align}
y_{i} & =f_{i}^{+}\left(\mb x_{\star}\right)-f_{i}^{-}\left(\mb x_{\star}\right)+\xi_{i} & ,\,i=1,\dotsc,n\,,\label{eq:dc-reg}
\end{align}
 where $\xi_{i}$s denote additive noise. Given the data \[\left(f_{i}^{+}\left(\cdot\right),f_{i}^{-}\left(\cdot\right),y_{i}\right)\,\quad i=1,\dotsc,n\,,\]
the goal is to accurately estimate $\mb x_{\star}$, up to the possible inherent ambiguities, by a computationally
tractable procedure. 

One can immediately notice the \emph{difference
of convex }(DC) structure\footnote{Sometimes this structure is referred to as \emph{convex-concave}, indicating the decomposition into the sum of a convex function and
a concave function.} in the observation model \eqref{eq:dc-reg}. %A
Many parametric regression problems
can be abstracted as \eqref{eq:dc-reg} due to richness of the set of DC functions \cite{Hartman1959Functions}; for instance, any smooth function can be expressed in the DC form using positive and negative semidefinite
parts of its Hessian. While it is evident from the considered form of the observed data, we emphasize that the DC decomposition of the observation function is assumed to be known and our proposed estimator relies on such a DC decomposition.

%A 
In the context of the model \eqref{eq:dc-reg}, the standard estimators
based on empirical risk minimization such as (nonlinear) least squares
would lead to nonconvex optimization problems that are generally
computationally hard. Thus, without making any assumption, our search
for a computationally efficient estimator for \eqref{eq:dc-reg} may
be futile. Of course, statistical assumptions are also necessary to make the estimation meaningful; the observations must convey (enough) information about the ground truth parameter. All of the assumptions we make are discussed in more detail in Section \ref{ssec:assumptions}.

Throughout we use the notation $\E_{\mc D}$ or $\mbb E_{\mc D^{n}}$
to denote the expectation with respect to a single or multiple observations.
Outer product of vectors is denoted by the binary operation $\otimes$.
Furthermore, $\left\lVert \cdot\right\rVert $, $\left\lVert \cdot\right\rVert _{\F}$,
and $\left\lVert \cdot\right\rVert _{\mr{op}}$ respectively denote
the usual Euclidean norm, Frobenius norm, and operator norm.

\subsection{\label{ssec:assumptions}Statistical and computational assumptions}
%A
In this section we describe the main statistical and computational assumptions we rely on in our analysis some of which were alluded to above. Stating some of these assumptions requires us to define certain parameters of the model for which we provide the motivations subsequently.

To avoid long expressions, for $i=1,\dotsc,n$, we define 
\begin{align}
q_{i}(\mb h) & \defeq\frac{1}{2}\left|\langle\nabla f_{i}^{+}(\mb x_{\star})-\nabla f_{i}^{-}(\mb x_{\star}),\mb h\rangle\right|\,,\label{eq:qi}
\end{align}
which are clearly nonnegative and positive homogeneous. Therefore,
by the triangle inequality, they also satisfy
\begin{align}
\left|q_{i}(\mb h)-q_{i}(\mb h')\right| & \le q_{i}(\mb h-\mb h')\,,\label{eq:triangle-like}
\end{align}
 for every pair of $\mb h,\mb h'\in\mbb R^{d}$. 
 
As it becomes
clear in the sequel, the central piece in our analysis is to establish
 a lower bound for the empirical process $\frac{1}{n}\sum_{i=1}^{n}q_{i}(\mb h)$ uniformly for a set of vectors $\mb h$. A crucial point in
our proof is that $\E_{\mc D}\left(q_i(\mb z)\right)$ is linear
 in $\left\lVert \mb z\right\rVert $. %A
 If $\E_{\mc D}(q_i(\mb z))$ had a different modulus of continuity and did not
admit a lower bound with linear growth in $\left\lVert \mb z\right\rVert $,
then a nontrivial lower bound for the mentioned empirical process
that holds uniformly in an arbitrarily small neighborhood of the origin
might not exist. The consequence would be an error bound
that does not vanish by removing the additive noise. These circumstances
are observed and well-understood, for instance, in the contexts of
\emph{ratio limit theorems} \citep{Gine2003Ratio,Gine2006Concentration},
the issue of a nontrivial \emph{version space }in learning problems
\citep{Mendelson2014Learning,Mendelson2015Learning}, and implicitly
in specific applications such one-bit compressed sensing and its generalizations
\citep{Plan2013Robust,Plan2016Generalized}.

We use the random function $q(\mb h)=\frac{1}{2}\left|\langle\nabla f^{+}(\mb x_{\star})-\nabla f^{-}(\mb x_{\star}),\mb h\rangle\right|$,
which has the same law as the functions $q_{i}$, to define a few
important quantities below.

\paragraph{Conditioning:}

Let $\mbb S^{d-1}$ denote the usual unit sphere in $\mbb R^{d}$.
Given $\mc S\subseteq\mbb S^{d-1}$, we define $\lambda_{\mc D}$
and $\varLambda_{\mc D}$ as

\begin{align}
\lambda_{\mc D} & \defeq\inf_{\mb z\in\mc S}\ \E_{\mc D}\left(q(\mb z)\right)\,,\label{eq:small-lambda}
\end{align}
and
\begin{align}
\varLambda_{\mc D} & \defeq\sup_{\mb z\in\mc S}\ \E_{\mc D}\left(q(\mb z)\right)\,.\label{eq:big-lambda}
\end{align}
The dependence of $\lambda_{\mc D}$ and $\varLambda_{\mc D}$ on
$\mc S$ will always be clear from the context, thus we do not make
this dependence explicit merely to simplify the notation. %A 
Our results will depend on the \emph{condition number} $\varLambda_\mc{D}/\lambda_\mc{D}$. In particular, it is important to have $\lambda_\mc{D}> 0$.

While generically $\mc S$ can be set to $\mbb S^{d-1}$ in the definitions \eqref{eq:small-lambda}
and \eqref{eq:big-lambda}, in some applications we may choose $\mc S$
to be a proper subset of $\mbb S^{d-1}$. This restriction helps us
avoiding a \emph{degeneracy} that leads to $\lambda_{\mc D}=0$ and
vacuous error bounds. An interesting example occurs in the \emph{bilinear
regression} problem discussed in Section \ref{sec:bilinear-regression}.\looseness=-1

%A
 Our proposed estimator, described in Section \ref{sec:main-result}, can be viewed as an approximation to 
\[
	\argmax_{\mb{x}} \ \E\left(\frac{1}{2}\langle\nabla f^+(\mb{x}_\star) + \nabla f^-(\mb{x}_\star),\mb{x}-\mb{x}_\star \rangle-\max\{f^+(\mb{x})-f^+(\mb{x}_\star),f^-(\mb{x})-f^-(\mb{x}_\star)\}\right)\,, 
\]
disregarding the additive constants in the objective function.
The importance of $\lambda_\mc{D}$ can be explained by inspecting the uniqueness of the above ``idealized'' estimator. By convexity of $f^\pm(\cdot)$ we have 
\begin{align*}
	f^\pm(\mb{x}) - f^\pm(\mb{x}_\star) &\ge\langle\nabla f^\pm(\mb{x}_\star),\mb{x}-\mb{x}_\star\rangle \,, 
\end{align*}
and thereby 
\begin{align*}
	\max\{f^+(\mb{x})-f^+(\mb{x}_\star),f^-(\mb{x})-f^-(\mb{x}_\star)\} & \ge \max\{\langle\nabla f^+(\mb{x}_\star),\mb{x}-\mb{x}_\star\rangle,\langle\nabla f^-(\mb{x}_\star),\mb{x}-\mb{x}_\star\rangle\}\,.
\end{align*}
Therefore, the objective function of the idealized estimator is dominated by 
\[-\frac{1}{2}\E\left(|\langle\nabla f^+(\mb{x}_\star) - \nabla f^-(\mb{x}_\star),\mb{x}-\mb{x}_\star \rangle|\right) \,.\]
The points $\mb{x}$ for which $f^+(\mb{x})-f^-(\mb{x}) = f^+(\mb{x}_\star) - f^-(\mb{x}_\star)$ almost surely are effectively equivalent to $\mb{x}_\star$. Thus, in view of \eqref{eq:small-lambda},  with  $\mc{S}\subseteq\mbb{S}^{d-1}$ being the complement of the directions from $\mb{x}_\star$ to its equivalents, having $\lambda_\mc{D} > 0$ guarantees that the idealized estimator can only be $\mb{x}=\mb{x}_\star$.

\paragraph{Regularity:}

For technical reasons we also need some regularity for the data distribution. %A
To exclude pathologically heavy-tailed data distributions we make the mild assumption that the (directional) second moment of $\nabla f^+(\mb{x}_\star) -\nabla f^-(\mb{x}_\star)$ is bounded from above by its corresponding (directional) first moment. This assumption can be made precise in terms of $q(\mb{z})$ as follows. For some constant $\eta_{\mc D}>1$ we assume that 
\begin{align}
\sqrt{\E_{\mc D}\left(q^{2}(\mb z)\right)} & \le\eta_{\mc D}\E_{\mc D}\left(q(\mb z)\right)\,,\label{eq:tail-weight}
\end{align}
holds for all $\mb z\in\mbb R^{d}$.  Furthermore,
with $\mb g$ denoting a standard normal random variable, we define
\begin{equation}
\varGamma_{\mc D}\defeq\sqrt{\E_{\mb g}\E_{\mc D}\left(q^{2}(\mb g)\right)}=\frac{1}{2}\sqrt{\E_{\mc D}\left(\left\lVert \nabla f^{+}(\mb x_{\star})-\nabla f^{-}(\mb x_{\star})\right\rVert ^{2}\right)}\,,\label{eq:sensitivity}
\end{equation}
 which is a measure of smoothness of the functions $q_{i}$ near the
origin. %A
The main factor in the sample complexity we establish is $\varGamma^2_\mc{D}/\varLambda^2_\mc{D}$ that can be interpreted as the \emph{effective dimension} of the problem since it is bounded by the ratio of the trace and the operator norm of the correlation matrix of $\nabla f^{+}(\mb x_{\star})-\nabla f^{-}(\mb x_{\star})$.

\paragraph{Approximation oracle:}

We assume an \emph{approximation oracle} is available that provides
a vector $\mb a_{0}\in\mbb R^{d}$ which, for some $\varepsilon\in\left(0,1\right]$,
obeys 
\begin{align}
\left\lVert \mb a_{0}-\frac{1}{2n}\sum_{i=1}^{n}\nabla f_{i}^{+}(\mb x_{\star})+\nabla f_{i}^{-}(\mb x_{\star})\right\rVert  & \le\frac{1-\varepsilon}{2}\lambda_{\mc D}\,.\label{eq:ApproxOracle}
\end{align}
Having access to the approximation oracle above is the strongest assumption
we make. This assumption could be excessive for prediction tasks where
the goal is merely accurate approximation of $f^{+}(\mb x_{\star})-f^{-}(\mb x_{\star})$
for the unseen data. However,  in this paper we are analyzing an estimation
task in which accurate estimation of $\mb x_{\star}$ is the goal rather than predicting $f^+(\mb{x}_\star) - f^-(\mb{x}_\star)$.
A standard approach to such estimation problems is to optimize an
\emph{empirical risk} that quantifies the consistency of any candidate
estimate with the observations. Because these risk functions are generally
nonconvex, accuracy guarantees for iterative estimation procedures
is often established assuming that they are \emph{initialized} at
a point, say $\mb x_{0}$, in a relatively small neighborhood of the
ground truth $\mb x_{\star}$ (i.e., $\mb x_{0}\approx\mb x_{\star}$).
The imposed bound \eqref{eq:ApproxOracle} can be derived from such
initialization conditions; e.g., if $\nabla f^{+}(\cdot)+\nabla f^{-}(\cdot)$
is a sufficiently smooth mapping, then $\mb x_{0}\approx\mb x_{\star}$
would imply $\mb a_{0}=\frac{1}{2n}\sum_{i=1}^{n}\nabla f_{i}^{+}(\mb x_{0})+\nabla f_{i}^{-}(\mb x_{0})\approx\frac{1}{2n}\sum_{i=1}^{n}\nabla f_{i}^{+}(\mb x_{\star})+\nabla f_{i}^{-}(\mb x_{\star})$.
Finally, if the vectors ${\nabla f^{+}(\mb x_{\star})+\nabla f^{-}(\mb x_{\star})}$
are sufficiently light-tailed in the sense of being bounded in a certain
\emph{Orlicz norm}\footnote{For a precise definition, interested readers are referred to \citep[Appendix A.1]{Koltchinskii2011Oracle}
and the references therein.}, then we can simply require 
\begin{align*}
\left\lVert \mb a_{0}-\frac{1}{2}\E_{\mc D}\left(\nabla f^{+}(\mb x_{\star})+\nabla f^{-}(\mb x_{\star})\right)\right\rVert  & \le\frac{1-\varepsilon'}{2}\lambda_{\mc D}\,,
\end{align*}
 and then resort to a matrix concentration inequality such as the
matrix Bernstein inequality \citep[Theorem 2.7]{Koltchinskii2011Oracle}
or the matrix Rosenthal inequality \citep{Chen2012Masked,Junge2013Noncommutative,Mackey2014Matrix}
to recover the condition \eqref{eq:ApproxOracle}. We do not attempt
to provide a general framework to address these details in this paper.
However, in the context of the bilinear regression problem, following the idea of  ``spectral initialization'' used in nonconvex methods (see, e.g., \cite{Netrapalli2013Phase, Candes2014Phase, Ling2018Regularized, Li2018Rapid, Ma2017Implicit}) we provide an explicit example for an implementable approximation
oracle in Section \ref{sec:bilinear-regression}.

The three assumptions stated above are primarily related to the statistical
model. We also make the following assumptions on the computational
model in order to provide a tractable method.

\paragraph{Computational assumptions:}

%A
As mentioned above, we emphasize that our approach requires the access to the DC decomposition
of the observation function. Computing such a decomposition
can be intractable in general (see, e.g., \citep{Ahmadi2017DC} and
references therein). However, assuming access to an efficiently computable
DC form is a reasonable compromise for creating a concrete computational
framework. In many applications the DC decomposition is
provided explicitly or is easy to compute. For instance, many statistical problems are concerned with observations of the form $y_i=\phi(\langle\mb{a}_{i}, \mb{x}_{\star}\rangle)$ for a certain nonlinear function $\phi : \mbb{R}\to\mbb{R}$ and data point $\mb{a}_i$. In some interesting instances, the desired DC decomposition is relatively easy to compute because it reduces to computing a DC decomposition of $\phi(\cdot)$ over $\mbb{R}$. Of course, as a natural requirement for implementing
optimization algorithms such as the first-order methods, we also need the components of the DC decomposition
(and their gradients) to be computable.

\section{\label{sec:main-result}The estimator and the main results}

Given $\mb a_{0}$, the output of the approximation oracle which obeys
\eqref{eq:ApproxOracle}, we formulate the estimator of $\mb x_{\star}$
as

\begin{align}
\widehat{\mb x} & \in\argmax_{\mb x}\ \langle\mb a_{0},\mb x\rangle-\frac{1}{n}\sum_{i=1}^{n}\max\left\{ f_{i}^{+}\left(\mb x\right)-y_{i},f_{i}^{-}\left(\mb x\right)\right\} \,,\label{eq:AnchoredERM}
\end{align}
which is a convex program that can be solved efficiently.

 %A
 Let us first demystify the formulation of the estimator by some intuitive explanations. Using the identity \[\max\{u,v\} = \frac{u+v+|u-v|}{2}\,,\] the objective function in \eqref{eq:AnchoredERM} can be expressed as 
\begin{align*}
	 & \langle{\mb{a}_0},\mb{x}\rangle -\frac{1}{n}\sum_{i=1}^{n}\max\left\{ f_{i}^{+}\left(\mb x\right)-y_{i},f_{i}^{-}\left(\mb x\right)\right\} \\ & =  -\frac{1}{2n}\sum_{i=1}^{n} \left(f_{i}^{+}\left(\mb x\right)+f_{i}^{-}\left(\mb x\right)-\langle{2\mb{a}_0},\mb{x}\rangle-y_{i}\right) -\frac{1}{2n}\sum_{i=1}^n|f_{i}^{+}\left(\mb x\right)-f_{i}^{-}\left(\mb x\right)-y_i|\,.
\end{align*}
Suppose that, instead of \eqref{eq:ApproxOracle}, we have $\mb{a}_0= \frac{1}{2n}\sum_{i=1}^{n}\nabla f_{i}^{+}(\mb x_{\star})+\nabla f_{i}^{-}(\mb x_{\star})$. Inspecting the first sum, it is evident that it is, up to additive constants, a Bregman divergence that admits $\mb{x}_\star$ as a minimizer. Furthermore, the second sum is expected to be minimized at a point close to $\mb{x}_\star$ because the observations obey $y_i\approx f_i^+(\mb{x}_\star)-f_i^-(\mb{x}_\star)$. Therefore, we can expect that the estimator $\widehat{\mb{x}}$ is not far from $\mb{x}_\star$. The Karush-Kuhn-Tucker (KKT) stationarity condition, explains that the approximation error of $\mb{a}_0$ in \eqref{eq:ApproxOracle} is tolerable because the the second sum contains nondifferentiable terms with (potantially) large subdifferentials.  Our analysis makes these intuitive explanations rigorous in an implicit manner.

With the definitions and assumptions stated in Section \ref{ssec:assumptions},
our main result is the following theorem that provides the sample
complexity for accuracy \eqref{eq:AnchoredERM} in a generic setting.
This theorem is a simple consequence of Proposition \ref{pro:master}
and Lemma \ref{lem:accuracy} stated subsequently.
\begin{thm}
\label{thm:generic}Given a set $\mc S\subseteq\mbb S^{d-1}$ and
parameter $\varepsilon\in\left(0,1\right)$, suppose \eqref{eq:small-lambda},
\eqref{eq:big-lambda}, \eqref{eq:tail-weight}, \eqref{eq:sensitivity},
and \eqref{eq:ApproxOracle} hold. Furthermore, for a solution $\widehat{\mb x}$
of \eqref{eq:AnchoredERM}, suppose that we have 
\begin{align}
\widehat{\mb x}-\mb x_{\star} & \in\left\lVert \widehat{\mb x}-\mb x_{\star}\right\rVert \mc S\,.\label{eq:restriction}
\end{align}
If the number of measurements obeys 
\begin{align*}
n & \ge C\max\left\{ \eta_{\mc D}^{2}\log\frac{2}{\delta},\:\frac{\varGamma_{\mc D}^{2}}{\lambda_{\mc D}\varLambda_{\mc D}}\right\} \frac{\varLambda_{\mc D}^{2}}{\lambda_{\mc D}^{2}}\eta_{\mc D}^{2}\varepsilon^{-4}\,,
\end{align*}
for a sufficiently large absolute constant $C>0$, then with probability
$\ge1-\delta$ we have
\begin{align*}
\left\lVert \widehat{\mb x}-\mb x_{\star}\right\rVert  & \le\frac{\frac{1}{n}\sum_{i=1}^{n}\left|\xi_{i}\right|}{\frac{1}{2}\lambda_{\mc D}}\,.
\end{align*}
\end{thm}
\begin{proof}
Given the lower bound on $n$, Proposition \ref{pro:master} below
together with \eqref{eq:small-lambda} guarantee that, with probability
$\ge1-\delta$, we have
\begin{align*}
\frac{1}{n}\sum_{i=1}^{n}q_{i}(\mb h) & \ge(1-\frac{1}{2}\varepsilon)\lambda_{\mc D} \,,
\end{align*}
 for all $\mb h\in\mc S$. The desired error bound follows immediately
from Lemma \ref{lem:accuracy} with $\mb x_{0}=\mb x_{\star}$ and
$\varepsilon_{0}=0$.
\end{proof}
%A
The condition \eqref{eq:restriction} in the theorem may appear unnatural at first. Clearly, the condition holds if we choose $\mc S=\mbb S^{d-1}$. However, the condition \eqref{eq:restriction} is imposed to address the situations where a set of equivalent ground
truth vectors $\mb x_{\star}$ exists and we only need to prove accuracy
with respect to the closest point in this set. This relaxed accuracy
requirement induces additional structure on $\widehat{\mb x}-\mb x_{\star}$ that should be considered to avoid the degeneracy at $\lambda_{\mc{D}}=0$. The sole purpose of \eqref{eq:restriction} is to capture the mentioned additional structures. The bilinear regression problem discussed below in Section \ref{sec:bilinear-regression}
is an example where it is important to have a nontrivial set $\mc S$. 

%A
Furthermore, with $\kappa \defeq \varLambda_\mc{D}/\lambda_\mc{D}$ and $d_\mr{eff}\defeq\varGamma^2_\mc{D}/\varLambda_\mc{D}^2$, the achievable sample complexity stated by the Theorem \ref{thm:generic} can be rewritten as
\[n\ge C \max\{\eta_\mc{D}^2 \log\frac{2}{\delta},\,\kappa\, d_\mr{eff}\}\kappa^2\eta_\mc{D}^2 \varepsilon^{-4}\,,\]
signifying the role of the effective dimension $d_\mr{eff} \le d$, and the conditioning $\kappa$ of the problem.

Under the assumptions stated in Section \ref{ssec:assumptions},
accuracy of the estimator \eqref{eq:AnchoredERM} can be reduced to
the existence of an appropriate uniform lower bound for the empirical
process $\frac{1}{n}\sum_{i=1}^{n}q_{i}(\mb h)$ as a function of
$\mb h$. The following Lemma \ref{lem:accuracy}, proved in Section
\ref{ssec:Proof-Accuracy-Lem}, provides the precise form of this
reduction.
\begin{lem}
\label{lem:accuracy} Let $\mb x_{0}$ be one of the possibly many
vectors equivalent to $\mb x_{\star}$ meaning that 
\begin{align*}
f^{+}(\mb x_{0})-f^{-}(\mb x_{0}) & =f^{+}(\mb x_{\star})-f^{-}(\mb x_{\star})\,,
\end{align*}
 almost surely. Given a set $\mc S\subseteq\mbb S^{d-1}$, recall
the definition \eqref{eq:small-lambda} and assume that an analog
of the condition \eqref{eq:ApproxOracle} with respect to $\mb x_{0}$
holds, namely, 
\begin{align}
\left\lVert \mb a_{0}-\frac{1}{2n}\sum_{i=1}^{n}\nabla f_{i}^{+}(\mb x_{0})+\nabla f_{i}^{-}(\mb x_{0})\right\rVert  & \le\frac{1-\varepsilon}{2}\lambda_{\mc D}\,,\label{eq:ApproxOracle-x0}
\end{align}
 for some constant parameter $\varepsilon\in(0,1)$. Furthermore,
suppose that \eqref{eq:restriction} holds and that for a certain
absolute constant $\varepsilon_{0}\in[0,1)$,
\begin{align}
\frac{1}{2n}\sum_{i=1}^{n}\left|\langle\nabla f_{i}^{+}(\mb x_{0})-\nabla f_{i}^{-}(\mb x_{0}),\mb h\rangle\right| & \ge(1-\frac{\varepsilon+\varepsilon_{0}}{2})\lambda_{\mc D} \,,\label{eq:empirical-process-LB}
\end{align}
holds for every $\mb h\in\mc S$. Then the estimate $\widehat{\mb x}$
obeys 
\[
\left\lVert \widehat{\mb x}-\mb x_{\star}\right\rVert \le\frac{\frac{1}{n}\sum_{i=1}^{n}\left|\xi_{i}\right|}{\frac{1-\varepsilon_{0}}{2}\lambda_{\mc D}}\,.
\]
\end{lem}
Again, in the generic case we choose $\mc S=\mbb S^{d-1}$, $\mb x_{0}=\mb x_{\star}$,
and $\varepsilon_{0}=0$ in Lemma \ref{lem:accuracy}. For the structured
problems mentioned above, however, with a nontrivial choice of $\mc S$
in \eqref{eq:restriction}, we may need to choose $\mb x_{0}\ne\mb x_{\star}$
and an appropriate $\varepsilon_{0}>0$.

Lemma \ref{lem:accuracy} provides an error bound that is proportional
to $\frac{1}{n}\sum_{i=1}^{n}\left|\xi_{i}\right|$. This dependence
is satisfactory for a deterministic noise model where we ought to
consider the worst-case scenarios. However, we may obtain improved
noise dependence for random noise models. In fact, simple modifications
in the proof of Lemma \ref{lem:accuracy} allow us to replace $\frac{1}{n}\sum_{i=1}^{n}\left|\xi_{i}\right|$
in the error bound by the maximum of the two expressions \[\left|\sup_{\mb x}\frac{1}{n}\sum_{i=1}^{n}\xi_{i}\bbone\left(f_{i}^{+}(\mb x)-f_{i}^{-}(\mb x)>f_{i}^{+}(\mb x_{\star})-f_{i}^{-}(\mb x_{\star})\right)\right|\]
and \[\left|\sup_{\mb x}\frac{1}{n}\sum_{i=1}^{n}-\xi_{i}\bbone\left(f_{i}^{+}(\mb x)-f_{i}^{-}(\mb x)>f_{i}^{+}(\mb x_{\star})-f_{i}^{-}(\mb x_{\star})+\xi_{i}\right)\right|\,.\]
These expressions may provide much tighter bounds when the noise is random with a well-behaved distribution. For instance, if $\xi_{1},\dotsc,\xi_{n}$
are i.i.d. zero-mean Gaussian random variables, the first expression
reduces to \emph{the Gaussian complexity} of the functions
$\bbone\left(f_{i}^{+}(\mb x)-f_{i}^{-}(\mb x)\!>\!f_{i}^{+}(\mb x_{\star})-f_{i}^{-}(\mb x_{\star})\right)$
which may be of order $n^{-1/2}$. To keep the exposition simple,
we focus on the deterministic noise model in this paper.

Clearly, to prove accuracy of \eqref{eq:AnchoredERM} through Lemma
\ref{lem:accuracy}, establishing an inequality of the form \eqref{eq:empirical-process-LB}
is crucial. Proposition \ref{pro:master} below can provide a guarantee
for such an inequality in the case $\mb x_{0}=\mb x_{\star}$ and
under the assumptions made in Section \ref{sec:introduction}.
\begin{prop}
\label{pro:master}Let $\varepsilon\in(0,1)$ be a constant parameter.
With the definitions \eqref{eq:qi}, \eqref{eq:small-lambda}, \eqref{eq:big-lambda},
\eqref{eq:tail-weight}, and \eqref{eq:sensitivity}, for any $\delta\in\left(0,1\right]$,
if for a sufficiently large absolute constant $C>0$ we have
\begin{align*}
n & \ge C\max\left\{ \eta_{\mc D}^{2}\log\frac{2}{\delta},\:\frac{\varGamma_{\mc D}^{2}}{\lambda_{\mc D}\varLambda_{\mc D}}\right\} \frac{\varLambda_{\mc D}^{2}}{\lambda_{\mc D}^{2}}\eta_{\mc D}^{2}\varepsilon^{-4}\,,
\end{align*}
 then with probability $\ge1-\delta$ the bound
\begin{align*}
\frac{1}{n}\sum_{i=1}^{n}q_{i}(\mb h) & \ge(1-\varepsilon)\E_{\mc D}(q(\mb h))\,,
\end{align*}
holds for every $\mb h\in\mc S$.
\end{prop}
The proof of this proposition is provided in Section \ref{sec:main-proofs}.

\subsection{Related work}

In a prior work \citep{Bahmani2017Anchored}, we considered the ``convex
regression'' model, a special case of \eqref{eq:dc-reg} with purely
convex nonlinearities (i.e., $f_{i}^{-}\equiv0$ and $f_{i}^{+}\equiv f_{i}$
for convex functions $f_{i}$). With a slightly weaker approximation
oracle that produces an \emph{anchor }$\mb a_{0}$ for which $\langle\mb a_{0},\mb x_{\star}\rangle/\left\lVert \mb a_{0}\right\rVert \left\lVert \mb x_{\star}\right\rVert $
is nonvanishing, statistical accuracy of estimation via the convex
program 
\begin{align*}
\argmax_{\mb x}\  & \langle\mb a_{0},\mb x\rangle\\
\text{subject to}\  & \frac{1}{n}\sum_{i=1}^{n}\max\left\{ f_{i}(\mb x)-y_{i},0\right\} \le\text{average noise}\,,
\end{align*}
is studied in \citep{Bahmani2017Anchored}. The effect of convex regularization
(e.g., $\ell_{1}$-regularization) in structured estimation (e.g.,
sparse estimation) is also considered and analyzed in \citep{Bahmani2017Anchored}.
Evidently, the solution of the convex program above is insensitive
to (positive) scaling of the anchor $\mb a_{0}$. The estimator \eqref{eq:AnchoredERM}
is, however, sensitive to the scaling of $\mb a_{0}$ which is a main
reason for the need for a slightly stronger approximation oracle in
this paper. An interesting example where the described convex regression
applies is the \emph{phase retrieval} problem that was previously
studied in \citep{Bahmani2016Phase,Bahmani2017Flexible,Goldstein2017Convex,Goldstein2018PhaseMax}.

As will be seen in Section \ref{sec:bilinear-regression}, bilinear
regression can be modeled by \eqref{eq:dc-reg} as well. Succinctly, the goal
in a bilinear regression problem is to recover signal components $\mb x^{(1)}$
and $\mb x^{(2)}$, up to the inevitable scaling ambiguity, from bilinear
observations of the form $\langle\mb a_{i}^{(1)},\mb x^{(1)}\rangle\langle\mb a_{i}^{(2)},\mb x^{(2)}\rangle$
for $i=1,\dotsc,n$. In the context of the closely related \textit{blind deconvolution} problem, solving
such a system of bilinear equations in the \textit{lifted domain}
through nuclear-norm minimization has been analyzed in \citep{Ahmed2014Blind}
and \citep{Bahmani2015Lifting}. Despite their accuracy guarantees,
the nuclear-norm minimization methods are practically not scalable
to large problem sizes which motivated the analysis of nonconvex
techniques (see, e.g., \citep{Li2018Rapid,Ling2018Regularized,Ma2017Implicit}).
Inspired by the results on the phase retrieval problem mentioned above,
\citep{Aghasi2017BranchHull} proposed and analyzed a convex program
for bilinear regression that operates in the natural space of the
signals, thereby avoiding the prohibitive computational cost of the
lifted convex formulations. Unlike the mentioned methods for phase
retrieval that only require a (directional) approximation of the ground
truth, the proposed estimator in \citep{Aghasi2017BranchHull} requires
the exact knowledge of the signs of all of the multiplied linear forms
$\langle\mb a_{i}^{(1)},\mb x_{\star}^{(1)}\rangle$ (or $\langle\mb a_{i}^{(2)},\mb x_{\star}^{(2)}\rangle$).
This requirement is rather strong and may severely limit the applicability
of the considered method. In Section \ref{sec:bilinear-regression},
we look into the problem of bilinear regression as a special case
of the general regression problem \eqref{eq:dc-reg}; under the common
Gaussian model for the measurement vectors we derive the sample complexity
of \eqref{eq:AnchoredERM} and explain an efficient method to construct
an admissible vector $\mb a_{0}$ only using the given observations.

\section{\label{sec:main-proofs}Main proofs}

There are various techniques under the umbrella of empirical process
theory that can be employed to establish Proposition \ref{pro:master}
and thereby Theorem \ref{thm:generic}. For instance, techniques relying
on the concepts of VC dimension \citep{Vapnik1971Uniform,Vapnik1998Statistical}
or Rademacher complexity \citep{Koltchinskii2000RademacherProcesses,Koltchinskii2001Rademacher,Bousquet2002Local}
including the small-ball method \citep{Koltchinskii2015Bounding,Mendelson2014Learning,Mendelson2015Learning}
that are primarily developed in the field of statistical learning
theory. %A
However, some techniques, such as the small-ball method, are designed particularly to  handle the type of heavy-tailed data we consider in our model. In this paper, we use another common technique, the PAC-Bayesian (or pseudo-Bayesian) method, that is suitable for heavy-tailed data. This method, proposed in
\citep{McAllester1999PAC-Bayesian}, has been used previously for
establishing various generalization bounds for classification \citep[see e.g.,][]{McAllester1999PAC-Bayesian,Langford2003PAC-Bayesian,Germain2009PAC-Beyesian}
and accuracy in regression problems \citep{Alquier2008PAC-Bayesian,Audibert2011Robust,Alquier2011PAC-Bayesian,Oliveira2016LowerTail,Catoni2017Dimension-free}.
The bounds obtained using this technique appear in different forms;
we refer the interested reader to the survey paper \citep{McAllester2013}
and the monograph \citep{Catoni2007PAC-Bayesian} for a broader view
of the related results and techniques. %A
Compared to the small-ball method, the PAC-Bayesian argument does not rely on the  \emph{symmetrization} \cite[Lemma 2.3.1]{vanderVaart1996Weak} and \emph{Rademacher contraction} \cite[Theorem 4.12]{Ledoux2013Probability} ideas and has a more elementary nature.

Our analysis below in Section \ref{ssec:PAC-Bayesian} parallels that
used in \citep{Oliveira2016LowerTail} which in turn was inspired
by \citep{Audibert2011Robust}. The technical tools we use can be
found in the PAC-Bayesian literature; we provide the proofs to make
the manuscript self-contained. We emphasize that the novelty of this
work is the general regression model \eqref{eq:dc-reg} and the computationally
efficient estimator \eqref{eq:AnchoredERM} rather than the methods
of analysis.

The core idea in the PAC-Bayes theory is the variational inequality\footnote{While this inequality is sometimes interpreted using the Fenchel\textendash Legendre
transform, it is simply a Jensen's inequality in disguise.} 
\begin{align}
\E_{\mb z\sim\mu}R(\mb z) & \le\log\E_{\mb z\sim\nu}\exp(R(\mb z))+D_{\mr{KL}}(\mu,\nu)\,,\label{eq:glorified-Jensen}
\end{align}
 where $D_{\mr{KL}}(\mu,\nu)=\E_{\mb z\sim\mu}\left(\log\frac{\d\mu(\mb z)}{\d\nu(\mb z)}\right)$
denotes the Kullback-Leibler divergence (or relative entropy) between
probability measures $\mu$ and $\nu$ with $\mu\ll\nu$. In PAC-Bayesian
analyses, the fact that this bound is deterministic and holds for
any probability measure $\mu\ll\nu$ is leveraged to control the supremum
of stochastic processes. In particular, for a stochastic process $R(\cdot)$
with the domain $\mc X$, we may \emph{approximate} $\sup_{\mb x\in\mc X}\ R(\mb x)$
by the supremum of $\sup_{\mu_{\mb x}:\mb x\in\mc X}\ \E_{\mb z\sim\mu_{\mb x}}R(\mb z)$
with respect to a certain set of probability measures $\mu_{\mb x}$ indexed by the elements of $\mc X$ (e.g., $\E_{\mb z\sim\mu_{\mb x}}(\mb z)=\mb x$).
Then, under some regularity conditions on the stochastic process,
the approximate bound can be converted to an exact bound.

\subsection{\label{ssec:PAC-Bayesian}A PAC-Bayesian proof of Proposition \ref{pro:master}}
We use the PAC-Bayesian analysis to establish Proposition \ref{pro:master},
the main ingredient in proving the accuracy of \eqref{eq:AnchoredERM}.
\begin{proof}[Proof of Proposition \ref{pro:master}]
 For $i=1,\dotsc,n$, let 
\begin{align}
w_{i}(\mb z) & \defeq\log\left(1-\left[\alpha q_{i}(\mb z)\right]_{\le1}+\frac{1}{2}\left[\alpha q_{i}(\mb z)\right]_{\le1}^{2}\right)\,,\label{eq:wi}
\end{align}
where $\left[u\right]_{\le1}\defeq\min\left(u,1\right)$ and $\alpha>0$
is a normalizing factor to be specified later. %A
As it becomes clear below, the function $w_i(\mb{z})$ should be viewed as an approximation for $\alpha q_i(\mb{z})$ that serves two purposes in the PAC-Bayesian argument. First, the use of the logarithm leads to \emph{cumulant generating functions} that can be relatively easily approximated by $-\E_{\mc{D}}(\alpha q_i(\mb{z}))$. Second, the use of the truncation legitimizes the evaluation of moment generating functions and also allows us to have bounded the deviations caused by the parameter perturbation in the PAC-Bayesian argument.

 Let $\gamma_{\mb h}$
denote the normal distribution with mean $\mb h$ and covariance $\sigma^{2}\mb I$
for a parameter $\sigma$. By \eqref{eq:glorified-Jensen}, for every
$\mb h\in\mc S\subseteq\mbb S^{d-1}$, we have

\begin{equation}
\begin{aligned}
 & \E_{\mb z\sim\gamma_{\mb h}}\left(\sum_{i=1}^{n}w_{i}(\mb z)-n\log\E_{\mc D}\exp\left(w_{1}(\mb z)\right)\right) \\
 & \le\log\E_{\mb z\sim\gamma_{\mb 0}}\exp\left(\sum_{i=1}^{n}w_{i}(\mb z)-n\log\E_{\mc D}\exp\left(w_{1}(\mb z)\right)\right) +\frac{\sigma^{-2}}{2}\,.
\end{aligned}
\label{eq:variational-bound}
\end{equation}
 Furthermore, by Markov's inequality with probability $\ge1-\delta/2$
we have 
\begin{align*}
 & \E_{\mb z\sim\gamma_{\mb 0}}\exp\left(\sum_{i=1}^{n}w_{i}(\mb z)-n\log\E_{\mc D}\exp\left(w_{1}(\mb z)\right)\right)\\
 & \le\frac{2}{\delta}\E_{\mc D^{n}}\E_{\mb z\sim\gamma_{\mb 0}}\exp\left(\sum_{i=1}^{n}w_{i}(\mb z)-n\log\E_{\mc D}\exp\left(w_{1}(\mb z)\right)\right)\\
 & =\frac{2}{\delta}\E_{\mb z\sim\gamma_{\mb 0}}\E_{\mc D^{n}}\exp\left(\sum_{i=1}^{n}w_{i}(\mb z)-n\log\E_{\mc D}\exp\left(w_{1}(\mb z)\right)\right)\\
 & =\frac{2}{\delta}\,,
\end{align*}
 where the exchange of expectations on the second line is valid as
the argument is a bounded function. Therefore, on the same event because
of \eqref{eq:variational-bound} for every $\mb h\in\mc S$ we have
\begin{align*}
\E_{\mb z\sim\gamma_{\mb h}}\left(\sum_{i=1}^{n}w_{i}(\mb z)-n\log\E_{\mc D}\exp\left(w_{1}(\mb z)\right)\right) & \le\frac{\sigma^{-2}}{2}+\log\frac{2}{\delta}\,,
\end{align*}
or equivalently 
\begin{align}
\frac{1}{n}\sum_{i=1}^{n}\E_{\mb z\sim\gamma_{\mb h}}w_{i}(\mb z) & \le\E_{\mb z\sim\gamma_{\mb h}}\left(\log\E_{\mc D}\exp(w_{1}(\mb z))\right)+\frac{\sigma^{-2}+2\log\frac{2}{\delta}}{2n}\,.\label{eq:bound-1}
\end{align}
The definition \eqref{eq:wi} and the facts that $-u\le\log(1-u+\frac{1}{2}u^{2})\le-u+\frac{1}{2}u^{2}$
and $1-\left[u\right]_{\le1}+\frac{1}{2}\left[u\right]_{\le1}^{2}\le1-u+\frac{1}{2}u^{2}$,
for all $u\ge0$, imply the bounds 
\begin{align*}
\E_{\mb z\sim\gamma_{\mb h}}w_{i}(\mb z) & \ge-\E_{\mb z\sim\gamma_{\mb h}}\left(\left[\alpha q_{i}(\mb z)\right]_{\le1}\right)\,,
\end{align*}
 and
\begin{align*}
\E_{\mb z\sim\gamma_{\mb h}}\log\E_{\mc D}\exp(w_{1}(\mb z)) & =\E_{\mb z\sim\gamma_{\mb h}}\log\left(\!1-\E_{\mc D}\left(\left[\alpha q_{1}(\mb z)\right]_{\le1}\right)+\frac{1}{2}\E_{\mc D}\left(\left[\alpha q_{1}(\mb z)\right]_{\le1}^{2}\right)\!\right)\\
 & \le-\E_{\mb z\sim\gamma_{\mb h}}\E_{\mc D}\left(\alpha q_{1}(\mb z)\right)+\frac{1}{2}\E_{\mb z\sim\gamma_{\mb h}}\E_{\mc D}\left(\alpha^{2}q_{1}^{2}(\mb z)\right)\,.
\end{align*}
Using \eqref{eq:triangle-like}, \eqref{eq:sensitivity}, and the
Cauchy-Schwarz inequality, we also have
\begin{align*}
\E_{\mb z\sim\gamma_{\mb h}}\E_{\mc D}\left(q_{1}(\mb z)\right) & \ge\E_{\mb z\sim\gamma_{\mb h}}\E_{\mc D}\left(q_{1}(\mb h)-q_{1}(\mb h-\mb z)\right)\\
 & =\E_{\mc D}\left(q_{1}(\mb h)\right)-\E_{\mb z\sim\gamma_{\mb 0}}\E_{\mc D}\left(q_{1}(\mb z)\right)\\
 & \ge\E_{\mc D}\left(q_{1}(\mb h)\right)-\sigma\varGamma_{\mc D}\,.
\end{align*}
Thus, it follows from \eqref{eq:bound-1} that

\begin{equation}
\begin{aligned}\frac{1}{n}\sum_{i=1}^{n}\frac{1}{\alpha}\E_{\mb z\sim\gamma_{\mb h}}\left(\left[\alpha q_{i}(\mb z)\right]_{\le1}\right) & \ge\E_{\mc D}\left(q_{1}(\mb h)\right)-\sigma\varGamma_{\mc D}\\
 & \phantom{\ge}-\frac{\alpha}{2}\E_{\mb z\sim\gamma_{\mb h}}\E_{\mc D}\left(q_{1}^{2}(\mb z)\right)-\frac{\sigma^{-2}+2\log\frac{2}{\delta}}{2\alpha n}\,.
\end{aligned}
\label{eq:bound-2}
\end{equation}
 Applying Lemmas \ref{lem:trucated2empirical} and \ref{lem:variance-term},
stated and proved in the appendix, to \eqref{eq:bound-2} shows that
for all $\mb h\in\mc S$, with probability $\ge1-\delta$, we have
\begin{align*}
 & \frac{1}{n}\sum_{i=1}^{n}q_{i}(\mb h)+\sigma\varGamma_{\mc D}+\frac{1}{\alpha}\sqrt{\frac{\log\frac{2}{\delta}}{2n}}\\
 & \ge\E_{\mc D}\left(q_{1}(\mb h)\right)-\sigma\varGamma_{\mc D}\\
 & \phantom{\ge}-\frac{\alpha}{2}\E_{\mb z\sim\gamma_{\mb h}}\E_{\mc D}\left(q_{1}^{2}(\mb z)\right)-\frac{\sigma^{-2}+2\log\frac{2}{\delta}}{2\alpha n}\\
 & \ge\E_{\mc D}\left(q_{1}(\mb h)\right)-\sigma\varGamma_{\mc D}-\frac{\alpha}{2}\left(\eta_{\mc D}\E_{\mc D}\left(q_{1}(\mb h)\right)+\sigma\varGamma_{\mc D}\right)^{2}\\
 & \phantom{\ge}-\frac{\sigma^{-2}+2\log\frac{2}{\delta}}{2\alpha n}\\
 & \ge\E_{\mc D}\left(q_{1}(\mb h)\right)-\sigma\varGamma_{\mc D}-\alpha\left(\eta_{\mc D}^{2}\left(\E_{\mc D}\left(q_{1}(\mb h)\right)\right)^{2}+\sigma^{2}\varGamma_{\mc D}^{2}\right)\\
 & \phantom{\ge}-\frac{\sigma^{-2}+2\log\frac{2}{\delta}}{2\alpha n}\,,
\end{align*}
By rearranging the terms, we reach at
\begin{align*}
\frac{1}{n}\sum_{i=1}^{n}q_{i}(\mb h) & \ge\E_{\mc D}\left(q_{1}(\mb h)\right)\\
 & \phantom{\ge}-2\sigma\varGamma_{\mc D}-\alpha\left(\eta_{\mc D}^{2}\left(\E_{\mc D}\left(q_{1}(\mb h)\right)\right)^{2}+\sigma^{2}\varGamma_{\mc D}^{2}\right)\\
 & \phantom{\ge}-\frac{1}{\alpha}\left(\sqrt{\frac{\log\frac{2}{\delta}}{2n}}+\frac{\sigma^{-2}+2\log\frac{2}{\delta}}{2n}\right)\,.
\end{align*}
 Recalling \eqref{eq:small-lambda}, \eqref{eq:big-lambda}, and \eqref{eq:sensitivity},
we can choose
\begin{align*}
\sigma & =\frac{\varepsilon\lambda_{\mc D}}{4\varGamma_{\mc D}}\,,
\end{align*}
 and 
\begin{align*}
\alpha & =\lambda_{\mc D}\left(\sqrt{\frac{\log\frac{2}{\delta}}{2n}}+\frac{16\varGamma_{\mc D}^{2}\lambda_{\mc D}^{-2}\varepsilon^{-2}+2\log\frac{2}{\delta}}{2n}\right)^{1/2}\left(\frac{\varLambda_{\mc D}}{\lambda_{\mc D}}\right)^{-1/2}\,,
\end{align*}
 to obtain 
\begin{align*}
 \frac{1}{n}\sum_{i=1}^{n}q_{i}(\mb h) & \ge\E_{\mc D}\left(q_{1}(\mb h)\right)-\frac{\varepsilon\lambda_{\mc D}}{2}\\
 &\phantom{\ge} -2\lambda_{\mc D}\left(\sqrt{\frac{\log\frac{2}{\delta}}{2n}}+\frac{16\varGamma_{\mc D}^{2}\lambda_{\mc D}^{-2}\varepsilon^{-2}+2\log\frac{2}{\delta}}{2n}\right)^{1/2}\left(\frac{\varLambda_{\mc D}}{\lambda_{\mc D}}\eta_{\mc D}^{2}+\frac{\varepsilon^{2}}{6}\right)^{1/2}\\
 & \ge\left(1-\frac{\varepsilon}{2}\right)\E_{\mc D}(q_{1}(\mb h))\\
 & \phantom{\ge} -2\lambda_{\mc D}\!\left(\sqrt{\frac{\log\frac{2}{\delta}}{2n}}+\frac{16\varGamma_{\mc D}^{2}\lambda_{\mc D}^{-2}\varepsilon^{-2}+2\log\frac{2}{\delta}}{2n}\right)^{1/2}\left(\frac{\varLambda_{\mc D}}{\lambda_{\mc D}}\eta_{\mc D}^{2}+\frac{\varepsilon^{2}}{6}\right)^{1/2}\,.
\end{align*}
 It is then straightforward to deduce 
\begin{align*}
\frac{1}{n}\sum_{i=1}^{n}q_{i}(\mb h) & \ge(1-\varepsilon)\,\E_{\mc D}\left(q_{1}(\mb h)\right)\,,
\end{align*}
assuming
\begin{align*}
n & \gtrsim\max\left\{ \eta_{\mc D}^{2}\log\frac{2}{\delta},\:\frac{\varGamma_{\mc D}^{2}}{\lambda_{\mc D}\varLambda_{\mc D}}\right\} \frac{\varLambda_{\mc D}^{2}}{\lambda_{\mc D}^{2}}\eta_{\mc D}^{2}\varepsilon^{-4}\,,
\end{align*}
 with a sufficiently large hidden constant.
\end{proof}

\subsection{\label{ssec:Proof-Accuracy-Lem}Proof of Lemma \ref{lem:accuracy}}

Below we provide a proof of Lemma \ref{lem:accuracy}.
\begin{proof}[Proof of Lemma \ref{lem:accuracy}]
By optimality of $\widehat{\mb x}$ in \eqref{eq:AnchoredERM} we
have

\begin{align*}
\frac{1}{n}\sum_{i=1}^{n}\max\left\{ f_{i}^{+}\left(\widehat{\mb x}\right)-y_{i},f_{i}^{-}\left(\widehat{\mb x}\right)\right\}  & \le\frac{1}{n}\sum_{i=1}^{n}\max\left\{ f_{i}^{+}\left(\mb x_{0}\right)-y_{i},f_{i}^{-}\left(\mb x_{0}\right)\right\} +\langle\mb a_{0},\widehat{\mb x}-\mb x_{0}\rangle\\
 & =\frac{1}{n}\sum_{i=1}^{n}f_{i}^{-}\left(\mb x_{0}\right)+\max\left\{ -\xi_{i},0\right\} +\langle\mb a_{0},\widehat{\mb x}-\mb x_{0}\rangle\,.
\end{align*}
For $i=1,\dotsc,n$ let $y_{\star i}=y_{i}-\xi_{i}=f_{i}^{+}(\mb x_{\star})-f_{i}^{-}(\mb x_{\star})=f_{i}^{+}(\mb x_{0})-f_{i}^{-}(\mb x_{0})$
and observe that 
\[
\max\left\{ f_{i}^{+}\left(\widehat{\mb x}\right)-y_{\star i},f_{i}^{-}\left(\widehat{\mb x}\right)\right\} \le\max\left\{ f_{i}^{+}\left(\widehat{\mb x}\right)-y_{i},f_{i}^{-}\left(\widehat{\mb x}\right)\right\} +\max\left\{ \xi_{i},0\right\} \,.
\]
 Therefore, we deduce that
\begin{align*}
 & \frac{1}{n}\sum_{i=1}^{n}\max\left\{ f_{i}^{+}\left(\widehat{\mb x}\right)-y_{\star i},f_{i}^{-}\left(\widehat{\mb x}\right)\right\} \\
 & \le\frac{1}{n}\sum_{i=1}^{n}f_{i}^{-}\left(\mb x_{0}\right)+\max\left\{ \xi_{i},0\right\} +\max\left\{ -\xi_{i},0\right\} +\langle\mb a_{0},\widehat{\mb x}-\mb x_{0}\rangle\\
 & =\frac{1}{n}\sum_{i=1}^{n}f_{i}^{-}\left(\mb x_{0}\right)+\left|\xi_{i}\right|+\langle\mb a_{0},\widehat{\mb x}-\mb x_{0}\rangle\,,
\end{align*}
or equivalently 
\begin{align*}
\frac{1}{n}\sum_{i=1}^{n}\max\left\{ f_{i}^{+}\left(\widehat{\mb x}\right)-f_{i}^{+}(\mb x_{0}),f_{i}^{-}\left(\widehat{\mb x}\right)-f_{i}^{-}\left(\mb x_{0}\right)\right\}  & \le\frac{1}{n}\sum_{i=1}^{n}\left|\xi_{i}\right|+\langle \mb a_{0},\widehat{\mb x}-\mb x_{0}\rangle \,.
\end{align*}
Invoking the assumption \eqref{eq:ApproxOracle-x0} and using Cauchy-Schwarz
inequality we can write 
\begin{align*}
& \frac{1}{n}\sum_{i=1}^{n}\max\left\{ f_{i}^{+}\left(\widehat{\mb x}\right)-f_{i}^{+}(\mb x_{0}),f_{i}^{-}\left(\widehat{\mb x}\right)-f_{i}^{-}\left(\mb x_{0}\right)\right\}\\
  & \le\frac{1}{n}\sum_{i=1}^{n}\left|\xi_{i}\right|+\langle\mb a_{0}-\frac{1}{2n}\sum_{i=1}^{n}\nabla f_{i}^{+}(\mb x_{0})+\nabla f_{i}^{-}(\mb x_{0}),\widehat{\mb x}-\mb x_{0}\rangle\\
 & \phantom{\le}+\langle\frac{1}{2n}\sum_{i=1}^{n}\nabla f_{i}^{+}(\mb x_{0})+\nabla f_{i}^{-}(\mb x_{0}),\widehat{\mb x}-\mb x_{0}\rangle\\
 & \le\frac{1}{n}\sum_{i=1}^{n}\left|\xi_{i}\right|+\frac{1-\varepsilon}{2}\lambda_{\mc D}\left\lVert \widehat{\mb x}-\mb x_{0}\right\rVert +\langle\frac{1}{2n}\sum_{i=1}^{n}\nabla f_{i}^{+}(\mb x_{0})+\nabla f_{i}^{-}(\mb x_{0}),\widehat{\mb x}-\mb x_{0}\rangle\,.
\end{align*}
Rearranging the terms gives the equivalent inequality 
\begin{align*}
 & \frac{1}{n}\sum_{i=1}^{n}\max\left\{ f_{i}^{+}\left(\widehat{\mb x}\right)-f_{i}^{+}(\mb x_{0}),f_{i}^{-}\left(\widehat{\mb x}\right)-f_{i}^{-}\left(\mb x_{0}\right)\right\} -\frac{1}{2}\langle\nabla f_{i}^{+}(\mb x_{0})+\nabla f_{i}^{-}(\mb x_{0}),\widehat{\mb x}-\mb x_{0}\rangle\\
 & \le\frac{1}{n}\sum_{i=1}^{n}\left|\xi_{i}\right|+\frac{1-\varepsilon}{2}\lambda_{\mc D}\left\lVert \widehat{\mb x}-\mb x_{0}\right\rVert \,.
\end{align*}
Observe that 
\begin{align*}
 & \max\left\{ f_{i}^{+}\left(\widehat{\mb x}\right)-f_{i}^{+}(\mb x_{0}),f_{i}^{-}\left(\widehat{\mb x}\right)-f_{i}^{-}\left(\mb x_{0}\right)\right\} -\frac{1}{2}\langle\nabla f_{i}^{+}(\mb x_{0})+\nabla f_{i}^{-}(\mb x_{0}),\widehat{\mb x}-\mb x_{0}\rangle\\
 & \ge\frac{1}{2}\left|\langle\nabla f_{i}^{+}(\mb x_{0})+\nabla f_{i}^{-}(\mb x_{0}),\widehat{\mb x}-\mb x_{0}\rangle\right|\,.
\end{align*}
Using the assumption that \eqref{eq:empirical-process-LB} holds,
we obtain 
\begin{align*}
 & \frac{1}{n}\sum_{i=1}^{n}\max\left\{ f_{i}^{+}\left(\widehat{\mb x}\right)-f_{i}^{+}(\mb x_{0}),f_{i}^{-}\left(\widehat{\mb x}\right)-f_{i}^{-}\left(\mb x_{0}\right)\right\} -\frac{1}{2}\langle\nabla f_{i}^{+}(\mb x_{0})+\nabla f_{i}^{-}(\mb x_{0}),\widehat{\mb x}-\mb x_{0}\rangle\\
 & \ge\frac{1}{2n}\sum_{i=1}^{n}\left|\langle\nabla f_{i}^{+}(\mb x_{0})+\nabla f_{i}^{-}(\mb x_{0}),\widehat{\mb x}-\mb x_{0}\rangle\right|\\
 & \ge\left(1-\frac{\varepsilon+\varepsilon_{0}}{2}\right)\lambda_{\mc D}\left\lVert \widehat{\mb x}-\mb x_{0}\right\rVert \,.
\end{align*}
Therefore, we conclude that 
\begin{align*}
\left(1-\frac{\varepsilon+\varepsilon_{0}}{2}\right)\lambda_{\mc D}\left\lVert \widehat{\mb x}-\mb x_{0}\right\rVert  & \le\frac{1}{n}\sum_{i=1}^{n}\left|\xi_{i}\right|+\frac{1-\varepsilon}{2}\lambda_{\mc D}\left\lVert \widehat{\mb x}-\mb x_{0}\right\rVert \,,
\end{align*}
which, since $\varepsilon_{0}\in[0,1)$, is equivalent to 
\begin{align*}
\left\lVert \widehat{\mb x}-\mb x_{0}\right\rVert  & \le\frac{\frac{1}{n}\sum_{i=1}^{n}\left|\xi_{i}\right|}{\frac{1-\varepsilon_{0}}{2}\lambda_{\mc D}}\,.
\end{align*}
\end{proof}

\section{\label{sec:bilinear-regression}Application to bilinear regression}
%A
In this section we apply the general result above to the problem of bilinear regression.
Suppose that the vectors $\mb x_{\star}^{(1)}$ and $\mb x_{\star}^{(2)}$
are observed through the bilinear measurements 
\begin{align}
y_{i} & =\langle\mb a_{i}^{(1)},\mb x_{\star}^{(1)}\rangle\langle\mb x{}_{\star}^{(2)},\mb a_{i}^{(2)}\rangle, & i=1,\dotsc,n\,,\label{eq:bilinear-eqs}
\end{align}
 with known vector pairs $(\mb a_{i}^{(1)},\mb a_{i}^{(2)})$. In bilinear regression, the goals is to recover $\mb{x}^{(1)}_\star$ and $\mb{x}_\star^{(2)}$ (up to the inherent ambiguities) from the above measurements.
 
%A
To apply our general framework, we introduce an equivalent formulation of the bilinear observations that is compatible with the DC observation model of \eqref{eq:dc-reg}. Let $\mb x_{\star}$ denote the concatenation of $\mb x_{\star}^{(1)}\in\mbb R^{d_{1}}\backslash\{\mb 0\}$
and $\mb x_{\star}^{(2)}\in\mbb R^{d_{2}}\backslash\{\mb 0\}$. Similarly,
for $i=1,\dotsc,n$ let $\mb a_{i}^{\pm}$
denote the concatenation of $\mb a_{i}^{(1)}$ and $\pm\mb a_{i}^{(2)}$. It is easy
to verify that the bilinear measurements above can also be expressed
in the form
\begin{align*}
y_{i} & =\frac{1}{4}\left|\langle \mb a_{i}^{+},\mb x_{\star}\rangle \right|^{2}-\frac{1}{4}\left|\langle \mb a_{i}^{-},\mb x_{\star}\rangle \right|^{2}\,,
\end{align*}
which is a special case of the DC observation model \eqref{eq:dc-reg}
with 
\begin{align}
f_{i}^{+}\left(\mb x\right) & =\frac{1}{4}\left|\langle \mb a_{i}^{+},\mb x\rangle \right|^{2} &  & \text{and} & f_{i}^{-}\left(\mb x\right) & =\frac{1}{4}\left|\langle \mb a_{i}^{-},\mb x\rangle \right|^{2}\,.\label{eq:quads}
\end{align}
The problem setup and additional notations are as follows. Denote the $\ell\times\ell$ identity matrix by $\mb I_{\ell}$. For $k=1,2$, let $\mb a_{1}^{(k)},\mb a_{2}^{(k)},\dotsc,\mb a_{n}^{(k)}$
be i.i.d. copies of $\mb a^{(k)}\sim\mr{Normal}(\mb 0,\mb I_{d_{k}})$ with $\mb{a}_i^{(1)}$ and $\mb{a}_i^{(2)}$ also drawn independently for all $1\le i\le n$.
Similar to the definition of $\mb{a}_i^\pm$s above, we also denote the concatenation of $\mb a^{(1)}$ and $\pm\mb a^{(2)}$
by $\mb a^{\pm}$. The functions $f^{\pm}$ are defined
analogous to $f_{i}^{\pm}$ with $\mb a^{\pm}$ replacing $\mb a_{i}^{\pm}$.
For brevity, we set $d=d_{1}+d_{2}$. Furthermore, some of the unspecified
constants in the derivations below are overloaded and may take different
values from line to line. For any vector $\mb x\in\mbb R^{d}$ with
partitions as $\mb x^{(1)}\in\mbb R^{d_{1}}$ and $\mb x^{(2)}\in\mbb R^{d_{2}}$,
we use the notation $\mb x^{-}$ to denote the concatenation of $\mb x^{(1)}$
and $-\mb x^{(2)}$.

Evidently, any reciprocal scaling of $\mb x_{\star}^{(1)}$ and $\mb x_{\star}^{(2)}$
is also consistent with the bilinear measurements \eqref{eq:bilinear-eqs}
and will be considered a valid solution. Throughout this section,
we choose $\mb x_{\star}$ to be a ``balanced'' solution meaning that
$\left\lVert \mb x_{\star}^{(1)}\right\rVert =\left\lVert \mb x_{\star}^{(2)}\right\rVert $.
Also, without loss of generality, we may assume $\langle\mb a_{0},\mb x_{\star}\rangle\ge0$.
The accuracy, however, is measured with respect to a closest consistent
solution 
\begin{align}
\widehat{\mb x}_{\star} & \in\argmin_{\mb x}\left\{ \left\lVert \widehat{\mb x}-\mb x\right\rVert \,:\,\mb x^{(1)}=t\mb x_{\star}^{(1)},\mb x^{(2)}=t^{-1}\mb x_{\star}^{(2)},t\in\mbb R\backslash\{0\}\right\} \,.\label{eq:closest-equivalent}
\end{align}

To state the accuracy guarantees for \eqref{eq:AnchoredERM} in the
described bilinear regression problem, we first bound the important
quantities given by \eqref{eq:small-lambda}, \eqref{eq:big-lambda},
\eqref{eq:tail-weight}, and \eqref{eq:sensitivity} for the restriction
set 
\begin{align}
\mc S & =\left\{ \mb z\in\mbb S^{d-1}\,:\,\left|\langle\mb z,\mb x_{\star}{}^{-}\rangle\right|\le\frac{1}{2}\left\lVert \mb x_{\star}\right\rVert \right\} \,.\label{eq:S-BL}
\end{align}
This choice of $\mc S$ allows us to find a nontrivial bound for
$\lambda_{\mc D}$ and it is important in the proof of Theorem \ref{thm:bilinear}.

\subsection{Quantifying \texorpdfstring{$\lambda_{\protect\mc D}$, $\varLambda_{\protect\mc D}$,
$\varGamma_{\protect\mc D}$, and $\eta_{\protect\mc D}$}{lambda_D, Lambda_D, Gamma_D, and eta_D} }

Let $\mb h\in\mc S$ be a vector partitioned into $\mb h^{(1)}\in\mbb R^{d_{1}}$
and $\mb h^{(2)}\in\mbb R^{d_{2}}$ . We can write 
\begin{align*}
\E_{\mc D}\left|\langle\nabla f^{+}(\mb x_{\star}))-\nabla f^{-}(\mb x_{\star})),\mb h\rangle\right| & =\E_{\mc D}\left|\langle\mb a^{(1)}\otimes\mb a^{(2)},\mb h^{(1)}\otimes\mb x_{\star}^{(2)}+\mb x_{\star}^{(1)}\otimes\mb h^{(2)}\rangle\right|\\
 & =\sqrt{\frac{2}{\pi}}\,\E_{\mc D}\left\lVert \left(\mb h^{(1)}\otimes\mb x_{\star}^{(1)}+\mb x_{\star}^{(1)}\otimes\mb h^{(2)}\right)\mb a^{(2)}\right\rVert \,.
\end{align*}
 Using the Cauchy-Schwarz inequality and Lemma \ref{lem:gaussian-norm}
in the appendix, respectively, we obtain 
\begin{align*}
\E_{\mc D}\left|\langle\mb a^{(1)}\otimes\mb a^{(2)},\mb h^{(1)}\otimes\mb x_{\star}^{(2)}+\mb x_{\star}^{(1)}\otimes\mb h^{(2)}\rangle\right| & \le\sqrt{\frac{2}{\pi}}\left\lVert \mb h^{(1)}\otimes\mb x_{\star}^{(2)}+\mb x_{\star}^{(1)}\otimes\mb h^{(2)}\right\rVert _{\F}\,,
\end{align*}
 and 
\begin{align*}
\E_{\mc D}\left|\langle\mb a^{(1)}\otimes\mb a^{(2)},\mb h^{(1)}\otimes\mb x_{\star}^{(2)}+\mb x_{\star}^{(1)}\otimes\mb h^{(2)}\rangle\right| & \ge\frac{2}{\pi}\left\lVert \mb h^{(1)}\otimes\mb x_{\star}^{(2)}+\mb x_{\star}^{(1)}\otimes\mb h^{(2)}\right\rVert _{\F}\,.
\end{align*}
 Observe that
\begin{align*}
\left\lVert \mb h^{(1)}\otimes\mb x_{\star}^{(2)}+\mb x_{\star}^{(1)}\otimes\mb h^{(2)}\right\rVert _{\F}^{2} & =\left\lVert \mb h^{(1)}\otimes\mb x_{\star}^{(2)}\right\rVert _{\F}^{2}+\left\lVert \mb x_{\star}^{(1)}\otimes\mb h^{(2)}\right\rVert _{\F}^{2}+2\langle\mb h^{(1)}\otimes\mb x_{\star}^{(2)},\,\mb x_{\star}^{(1)}\otimes\mb h^{(2)}\rangle\\
 & =\left\lVert \mb h^{(1)}\right\rVert ^{2}\left\lVert \mb x_{\star}^{(2)}\right\rVert ^{2}+\left\lVert \mb h^{(2)}\right\rVert ^{2}\left\lVert \mb x_{\star}^{(1)}\right\rVert ^{2}+2\langle\mb h^{(1)},\mb x_{\star}^{(1)}\rangle\langle\mb h^{(2)},\mb x_{\star}^{(2)}\rangle\\
 & =\langle\mb X_{\star},\mb h\otimes\mb h\rangle\,,
\end{align*}
where 
\begin{align*}
\mb X_{\star} & =\left[\begin{array}{cc}
\left\lVert \mb x_{\star}^{(2)}\right\rVert ^{2}\mb I_{d_{1}} & \mb x_{\star}^{(1)}\otimes\mb x_{\star}^{(2)}\\
\mb x_{\star}^{(2)}\otimes\mb x_{\star}^{(1)} & \left\lVert \mb x_{\star}^{(1)}\right\rVert ^{2}\mb I_{d_{2}}
\end{array}\right]\\
 & =\frac{1}{2}\left\lVert \mb x_{\star}\right\rVert ^{2}\mb I+\frac{1}{2}\mb x_{\star}\otimes\mb x_{\star}-\frac{1}{2}\mb x_{\star}^{-}\otimes\mb x_{\star}^{-}\,.
\end{align*}
 Thus, we have 
\begin{equation}
\frac{1}{\pi}\sqrt{\langle\mb X_{\star},\mb h\otimes\mb h\rangle}\le\frac{1}{2}\E_{\mc D}\left|\langle\nabla f^{+}(\mb x_{\star})-\nabla f^{-}(\mb x_{\star}),\mb h\rangle\right|\le\frac{1}{\sqrt{2\pi}}\sqrt{\langle\mb X_{\star},\mb h\otimes\mb h\rangle}\,.\label{eq:bilinear-Eqh}
\end{equation}
Since $\mb h\in\mc S$, by definition $\left|\langle\mb h,\mb x_{\star}^{-}\rangle\right|\le\frac{1}{2}\left\lVert \mb x_{\star}\right\rVert $,
and it is easy to verify that 
\[
\frac{3}{8}\left\lVert \mb x_{\star}\right\rVert ^{2}\le\langle\mb X_{\star},\mb h\otimes\mb h\rangle\le\left\lVert \mb x_{\star}\right\rVert ^{2}\,.
\]
 Therefore, \eqref{eq:bilinear-Eqh} implies that 
\[
\frac{\sqrt{6}}{4\pi}\left\lVert \mb x_{\star}\right\rVert \le\frac{1}{2}\E_{\mc D}\left|\langle\nabla f^{+}(\mb x_{\star})-\nabla f^{-}(\mb x_{\star}),\mb h\rangle\right|\le\frac{1}{\sqrt{2\pi}}\left\lVert \mb x_{\star}\right\rVert \,,
\]
which also means 
\begin{equation}
\frac{\sqrt{6}}{4\pi}\left\lVert \mb x_{\star}\right\rVert \le\lambda_{\mc D}\le\varLambda_{\mc D}\le\frac{1}{\sqrt{2\pi}}\left\lVert \mb x_{\star}\right\rVert \,.\label{eq:lambda-Lambda-BL}
\end{equation}
 Note that without the restriction of $\mb h$ to the prescribed set
$\mc S$ in \eqref{eq:S-BL}, we could have had $\lambda_{\mc D}=0$,
which leads to vacuous bounds.

We can also evaluate $\varGamma_{\mc D}$ as 
\begin{align}
\varGamma_{\mc D} & =\frac{1}{2}\sqrt{\E_{\mc D}\left(\left\lVert \nabla f^{+}(\mb x_{\star})-\nabla f^{-}(\mb x_{\star})\right\rVert ^{2}\right)}\nonumber \\
 & =\frac{1}{2}\sqrt{\E_{\mc D}\left(\left\lVert \mb a^{(1)}\langle\mb a^{(2)},\mb x_{\star}^{(2)}\rangle\right\rVert ^{2}+\left\lVert \mb a^{(2)}\langle\mb a^{(1)},\mb x_{\star}^{(1)}\rangle\right\rVert ^{2}\right)}\nonumber \\
 & =\frac{1}{2}\sqrt{\E_{\mc D}\left(d_{1}\left|\langle\mb a^{(2)},\mb x_{\star}^{(2)}\rangle\right|^{2}+d_{2}\left|\langle\mb a^{(1)},\mb x_{\star}^{(1)}\rangle\right|^{2}\right)}\nonumber \\
 & =\frac{1}{2}\sqrt{d_{1}\left\lVert \mb x_{\star}^{(2)}\right\rVert ^{2}+d_{2}\left\lVert \mb x_{\star}^{(1)}\right\rVert ^{2}}\nonumber \\
 & =\frac{1}{4}\sqrt{d}\left\lVert \mb x_{\star}\right\rVert \,.\label{eq:Gamma-BL}
\end{align}
Furthermore, using the lower bound in \eqref{eq:bilinear-Eqh}, we
can write
\begin{align*}
\E_{\mc D}\left(\left|\langle\nabla f^{+}(\mb x_{\star})-\nabla f^{-}(\mb x_{\star}),\mb h\rangle\right|^{2}\right) & =\E_{\mc D}\left(\left|\langle\mb h^{(1)},\mb a^{(1)}\rangle\langle\mb a^{(2)},\mb x_{\star}^{(2)}\rangle+\langle\mb x_{\star}^{(1)},\mb a^{(1)}\rangle\langle\mb a^{(2)},\mb h^{(2)}\rangle\right|^{2}\right)\\
 & =\left\lVert \mb h^{(1)}\right\rVert ^{2}\left\lVert \mb x_{\star}^{(2)}\right\rVert ^{2}+2\langle\mb h^{(1)},\mb x_{\star}^{(1)}\rangle\langle\mb h^{(2)},\mb x_{\star}^{(2)}\rangle+\left\lVert \mb h^{(2)}\right\rVert ^{2}\left\lVert \mb x_{\star}^{(1)}\right\rVert ^{2}\\
 & =\left\lVert \mb h^{(1)}\otimes\mb x_{\star}^{(2)}+\mb x_{\star}^{(1)}\otimes\mb h^{(2)}\right\rVert _{\F}^{2}\\
 & =\langle\mb X_{\star},\mb h\otimes\mb h\rangle\\
 & \le\frac{\pi^{2}}{4}\left(\E_{\mc D}\left(\left|\langle\nabla f^{+}(\mb x_{\star})-\nabla f^{-}(\mb x_{\star}),\mb h\rangle\right|\right)\right)^{2}\,.
\end{align*}
Thus, we are guaranteed to have 
\begin{align}
\eta_{\mc D} & \le\frac{\pi}{2}\,.\label{eq:eta-BL}
\end{align}

\subsection{Accuracy guarantee}

To prove accuracy of \eqref{eq:AnchoredERM} in the considered bilinear
regression problem, we need to apply Lemma \ref{lem:accuracy} with
$\widehat{\mb x}_{\star}$ given by \eqref{eq:closest-equivalent}
as the reference ground truth. Therefore, we also use Proposition
\ref{pro:master} with a nontrivial restriction set $\mc S$ in our
analysis to establish an inequality of the form \eqref{eq:empirical-process-LB}.
Because $\widehat{\mb x}_{\star}$ depends on the observations, however,
it cannot be used as the reference ground truth in Proposition \ref{pro:master}.
Lemma \ref{lem:conversion-lemma} in the appendix shows that the bound
obtained using Proposition \ref{pro:master}, with the balanced ground
truth (i.e., $\mb x_{\star}$) as the reference point and the restrictions
set \eqref{eq:S-BL}, can be extended to the cases where other equivalent
solutions are considered as the reference ground truth.

For any $t\in\mbb R\backslash0$, let $D_{t}:\mbb R^{d}\to\mbb R^{d}$
be the reciprocal scaling operator described by 
\begin{align*}
D_{t}(\mb x) & =\left[\begin{array}{c}
t\mb x^{(1)}\\
t^{-1}\mb x^{(2)}
\end{array}\right]\,,
\end{align*}
where $\mb x$ is the concatenation of $\mb x^{(1)}\in\mbb R^{d_{1}}$
and $\mb x^{(2)}\in\mbb R^{d_{2}}$. Furthermore, for $\theta\in\left[0,1\right]$,
we define the cone $\mc K_{t,\theta}$ as 
\begin{align}
\mc K_{t,\theta} & \defeq\left\{ \mb h\,:\,\left|\langle D_{t}(\mb x_{\star}^{-}),\mb h\rangle\right|\le\theta\left\lVert D_{t}(\mb x_{\star}^{-})\right\rVert \left\lVert \mb h\right\rVert \right\} \,.\label{eq:Kt,theta}
\end{align}
This specific choice of the cone $\mc K_{t,\theta}$ is important
for the following reason: If, for some $t_{\text{opt}}\in\mbb R\backslash\{0\}$,
$\widehat{\mb x}_{\star}=D_{t_{\text{opt}}}(\mb x_{\star})$ is the
solution described by \eqref{eq:closest-equivalent}, then elementary
calculus shows that
\begin{align*}
\langle D_{t_{\text{opt}}}(\mb x_{\star}^{-}),\widehat{\mb x}-\widehat{\mb x}_{\star}\rangle & =0\,,
\end{align*}
which means that $\widehat{\mb x}-\widehat{\mb x}_{\star}\in\mc K_{t_{\text{opt}},0}$.
Leveraging this property we can show that $D_{t_{\text{opt}}^{-1}}(\widehat{\mb x}-\widehat{\mb x}_{\star})\in\mc K_{1,\frac{1}{2}}$
which allows us to invoke Lemma \ref{lem:conversion-lemma}.

The following theorem establishes the sample complexity of \eqref{eq:AnchoredERM}
for exact recovery in the noiseless bilinear regression problem. 
\begin{thm}[bilinear regression]
\label{thm:bilinear} We observe $n$ noiseless bilinear measurements
\eqref{eq:bilinear-eqs} corresponding to the functions $f_{i}^{\pm}$
described by \eqref{eq:quads}. Suppose that \eqref{eq:ApproxOracle}
holds for some $\varepsilon\in[7/8,1)$. If the number of measurements
obeys
\begin{align}
n & \gtrsim\varepsilon^{-4}\max\left\{ d,\log\frac{8}{\delta}\right\} \,,\label{eq:sample-complexity-BL}
\end{align}
with a sufficiently large hidden constant, then with probability $\ge1-\delta$,
the solution to \eqref{eq:AnchoredERM} coincides with $\widehat{\mb x}_{\star}$
given by \eqref{eq:closest-equivalent}.
\end{thm}
\begin{proof}
Because of \eqref{eq:sample-complexity-BL}, we may assume 
\begin{align}
C\max\left\{ \sqrt{\frac{d}{n}}+\sqrt{\frac{\log\frac{8}{\delta}}{n}},\left(\sqrt{\frac{d}{n}}+\sqrt{\frac{\log\frac{8}{\delta}}{n}}\right)^{2}\right\}  & \le\frac{1}{9}\,,\label{eq:large-n}
\end{align}
where $C>0$ is the constant in Lemma \ref{lem:xstar-xhatstar}. With
$\varepsilon'=6\varepsilon-5$, it follows from \eqref{eq:a0-nabla-xhatstar}
in Lemma \ref{lem:xstar-xhatstar} that 
\begin{align}
\left\lVert \mb a_{0}-\frac{1}{2n}\sum_{i=1}^{n}\nabla f_{i}^{+}(\widehat{\mb x}_{\star})+\nabla f_{i}^{-}(\widehat{\mb x}_{\star})\right\rVert  & \le6\left\lVert \mb a_{0}-\frac{1}{2n}\sum_{i=1}^{n}\nabla f_{i}^{+}(\mb x_{\star})+\nabla f_{i}^{-}(\mb x_{\star})\right\rVert \nonumber \\
 & \le3(1-\varepsilon)\lambda_{\mc D}=\frac{1-\varepsilon'}{2}\lambda_{\mc D}\,,\label{eq:ApproxOracle-BL}
\end{align}
 holds with probability $\ge1-\delta/2$. 

Furthermore, the approximations \eqref{eq:lambda-Lambda-BL}, \eqref{eq:Gamma-BL},
and \eqref{eq:eta-BL} show that because of \eqref{eq:sample-complexity-BL},
Proposition \ref{pro:master}, with $\mb x_{\star}$ taken as the
reference ground truth, ensures 
\begin{align*}
\frac{1}{n}\sum_{i=1}^{n}\left|\langle\nabla f_{i}^{+}(\mb x_{\star})-\nabla f_{i}^{-}(\mb x_{\star}),\mb h\rangle\right| & \ge(1-\frac{1}{2}\varepsilon)\E_{\mc D}\left(\left|\langle\nabla f^{+}(\mb x_{\star})-\nabla f^{-}(\mb x_{\star}),\mb h\rangle\right|\right)\,,
\end{align*}
to hold for all $\mb h\in\mc S=\mbb S^{d-1}\cap\mc K_{1,\frac{1}{2}}$
with probability $\ge1-\delta/2$. On the same event, if $t_{\text{opt}}$,
defined as above through $\widehat{\mb x}_{\star}=D_{t_{\text{opt}}}(\mb x_{\star})$,
obeys $\sqrt{2/3}\le\left|t_{\text{opt}}\right|\le\sqrt{3/2}$, then
Lemma \ref{lem:conversion-lemma} implies that 
\begin{align*}
\frac{1}{n}\sum_{i=1}^{n}\left|\langle\nabla f_{i}^{+}(\widehat{\mb x}_{\star})-\nabla f_{i}^{-}(\widehat{\mb x}_{\star}),\mb h\rangle\right| & \ge(1-\frac{1}{2}\varepsilon)\E_{\mc D}\left(\left|\langle\nabla f^{+}(\widehat{\mb x}_{\star})-\nabla f^{-}(\widehat{\mb x}_{\star}),\mb h\rangle\right|\right)\\
 & =(1-\frac{1}{2}\varepsilon)\E_{\mc D}\left(\left|\langle\nabla f^{+}(\mb x_{\star})-\nabla f^{-}(\mb x_{\star}),D_{t_{\text{opt}}^{-1}}(\mb h)\rangle\right|\right)\,,
\end{align*}
for all $\mb h\in\mbb S^{d-1}\cap\mc K_{t_{\text{opt}},0}$. Note
that the expectations on the right-hand side are only with respect
to $f^{\pm}$; the vector $\widehat{\mb x}_{\star}$ and the scalar
$t_{\text{opt}}$ should be treated as deterministic variables. Using
\eqref{eq:small-lambda} we obtain 
\begin{align*}
\frac{1}{2n}\sum_{i=1}^{n}\left|\langle\nabla f_{i}^{+}(\widehat{\mb x}_{\star})-\nabla f_{i}^{-}(\widehat{\mb x}_{\star}),\mb h\rangle\right| & \ge(1-\frac{1}{2}\varepsilon)\lambda_{\mc D}\left\lVert D_{t_{\text{opt}}^{-1}}(\mb h)\right\rVert \,.
\end{align*}
 Therefore, the bound $\left\lVert D_{t_{\text{opt}}^{-1}}(\mb h)\right\rVert ^{2}=t_{\text{opt}}^{-2}\left\lVert \mb h^{(1)}\right\rVert ^{2}+t_{\text{opt}}^{2}\left\lVert \mb h^{(2)}\right\rVert ^{2}\ge\min\left\{ t_{\text{opt}}^{-2},t_{\text{opt}}^{2}\right\} \left\lVert \mb h\right\rVert ^{2}$
and the choice of $\varepsilon'=6\varepsilon-5$ made above yield
\begin{align}
\frac{1}{2n}\sum_{i=1}^{n}\left|\langle\nabla f_{i}^{+}(\widehat{\mb x}_{\star})-\nabla f_{i}^{-}(\widehat{\mb x}_{\star}),\mb h\rangle\right| & \ge(1-\frac{5+\varepsilon'}{12})\min\left\{ \left|t_{\text{opt}}^{-1}\right|,\left|t_{\text{opt}}\right|\right\} \lambda_{\mc D}\left\lVert \mb h\right\rVert \,.\label{eq:xhatstar-LB}
\end{align}

It only remains to bound $\left|t_{\text{opt}}\right|$ appropriately,
not only to approximate $\min\left\{ \left|t_{\text{opt}}^{-1}\right|,\left|t_{\text{opt}}\right|\right\} $,
but also to satisfy the condition $\sqrt{2/3}\le\left|t_{\text{opt}}\right|\le\sqrt{3/2}$
used previously. First, we show that $\left\lVert \widehat{\mb x}_{\star}-\mb x_{\star}\right\rVert $
is small through Lemma \ref{lem:xstar-xhatstar}. Note that the previous
application of Lemma \ref{lem:xstar-xhatstar}, in which we had \eqref{eq:large-n},
also guarantees 
\begin{align*}
\left\lVert \widehat{\mb x}_{\star}-\mb x_{\star}\right\rVert  & \le\frac{48}{5}\left\lVert \mb a_{0}-\frac{1}{2n}\sum_{i=1}^{n}\nabla f_{i}^{+}(\mb x_{\star})+\nabla f_{i}^{-}(\mb x_{\star})\right\rVert \le\frac{4}{5}(1-\varepsilon')\lambda_{\mc D}\,.
\end{align*}
 Therefore, using the upper bound in \eqref{eq:lambda-Lambda-BL},
we get 
\begin{align*}
\left\lVert \widehat{\mb x}_{\star}-\mb x_{\star}\right\rVert  & \le\frac{8}{25}(1-\varepsilon')\left\lVert \mb x_{\star}\right\rVert 
\end{align*}
 Because $\widehat{\mb x}_{\star}=D_{t_{\text{opt}}}(\mb x_{\star})$,
and $\mb x_{\star}$ is balanced, we also have
\begin{align*}
\left\lVert \widehat{\mb x}_{\star}-\mb x_{\star}\right\rVert ^{2} & =\left|t_{\text{opt}}-1\right|^{2}\left\lVert \mb x_{\star}^{(1)}\right\rVert ^{2}+\left|t_{\text{opt}}^{-1}-1\right|^{2}\left\lVert \mb x_{\star}^{(2)}\right\rVert ^{2}\\
 & \ge\frac{1}{2}\max\left\{ \left|t_{\text{opt}}^{-1}-1\right|^{2},\left|t_{\text{opt}}-1\right|^{2}\right\} \left\lVert \mb x_{\star}\right\rVert ^{2}\,,
\end{align*}
which together with the previous inequality imply 
\begin{align*}
\max\left\{ \left|t_{\text{opt}}^{-1}-1\right|,\left|t_{\text{opt}}-1\right|\right\}  & \le\frac{1}{2}(1-\varepsilon')\,.
\end{align*}
Therefore, we obtain
\begin{align*}
\min\left\{ \left|t_{\text{opt}}^{-1}\right|,\left|t_{\text{opt}}\right|\right\}  & =\left(\max\left\{ \left|t_{\text{opt}}^{-1}\right|,\left|t_{\text{opt}}\right|\right\} \right)^{-1}\\
 & \ge\left(1+\max\left\{ \left|t_{\text{opt}}^{-1}-1\right|,\left|t_{\text{opt}}-1\right|\right\} \right)^{-1}\\
 & \ge\left(1+\frac{1}{2}(1-\varepsilon')\right)^{-1}\\
 & \ge\left(1-\frac{5+\varepsilon'}{12}\right)^{-1}\left(1-\varepsilon'\right)\,,
\end{align*}
where the fourth line holds since $1/4\!\le\!\varepsilon'=6\varepsilon-5\!\le\!1$.
Using the derived bound in \eqref{eq:xhatstar-LB} yields
\begin{align*}
\frac{1}{2n}\sum_{i=1}^{n}\left|\langle\nabla f_{i}^{+}(\widehat{\mb x}_{\star})-\nabla f_{i}^{-}(\widehat{\mb x}_{\star}),\mb h\rangle\right| & \ge(1-\varepsilon')\lambda_{\mc D}\left\lVert \mb h\right\rVert \,.
\end{align*}
Hence, in view of \eqref{eq:ApproxOracle-x0}, we may invoke Lemma
\ref{lem:accuracy} with $\mb x_{0}=\widehat{\mb x}_{\star}$, $\varepsilon'$
in place of $\varepsilon$, and $\varepsilon_{0}=\varepsilon'$, and
prove the exact recovery (i.e., $\widehat{\mb x}=\widehat{\mb x}_{\star}$),
which occurs with probability $\ge1-\delta$. 
\end{proof}

\subsection{\label{ssec:Approx-Oracle}Approximation oracle}

We provide a computationally tractable procedure that can serve as
the approximation oracle discussed in Section \ref{ssec:assumptions}
and requires no information other than the given measurements \eqref{eq:bilinear-eqs}.
%A
This approach basically follows the idea of ``spectral initialization'' used for the nonconvex phase retrieval and blind deconvolution methods \cite{Netrapalli2013Phase,Candes2014Phase, Li2018Rapid}; refinements of this approach can be found in \cite{Chen2015Solving,Mondelli2018Fundamental,Luo2018Optimal} and references therein. We use the measurements to find an approximation $\mb a_{0}$ of $\mb x_{\star}/2$
and show, by Lemma \ref{lem:gradient-concentration}, that $\mb x_{\star}/2$
itself is an approximation for $\frac{1}{2n}\sum_{i=1}^{n}\nabla f_{i}^{+}(\mb x_{\star})+\nabla f_{i}^{-}(\mb x_{\star})$.

Let $\lambda_{\max}$ and $\mb v_{\max}$ be respectively the leading
eigenvalue and eigenvector of
\begin{align*}
\mb S_{n} & \defeq\frac{1}{2n}\sum_{i=1}^{n}y_{i}\left(\mb a_{i}^{+}\otimes\mb a_{i}^{+}-\mb a_{i}^{-}\otimes\mb a_{i}^{-}\right)\,.
\end{align*}
 The fact that $\mb S_{n}$ has an all-zero diagonal and is symmetric
ensures that $\lambda_{\max}\!\ge\!0$. We show that 
\begin{equation}
\mb a_{0}=\left(\frac{\lambda_{\max}}{2}\right)^{1/2}\mb v_{\max}\label{eq:a0-bilinear}
\end{equation}
 meets the required condition \eqref{eq:ApproxOracle} with high probability.
To this end, first we show that $\mb S_{n}$ is well-concentrated
around its expectation. Observe that $\langle\mb a_{i}^{-},\mb x_{\star}\rangle=\langle\mb a_{i}^{+},\mb x_{\star}^{-}\rangle$
and similarly $\langle\mb a_{i}^{-},\mb x_{\star}^{-}\rangle=\langle\mb a_{i}^{+},\mb x_{\star}\rangle$.
Thus, we obtain
\begin{align*}
\E_{\mc D}\mb S_{n} & =\E_{\mc D}\left(\frac{1}{8}\left(\left|\langle\mb a_{i}^{+},\mb x_{\star}\rangle\right|^{2}-\left|\langle\mb a_{i}^{-},\mb x_{\star}\rangle\right|^{2}\right)\left(\mb a_{i}^{+}\otimes\mb a_{i}^{+}-\mb a_{i}^{-}\otimes\mb a_{i}^{-}\right)\right)\\
 & =\frac{1}{8}\E_{\mc D}\left(\left|\langle\mb a_{i}^{+},\mb x_{\star}\rangle\right|^{2}\mb a_{i}^{+}\otimes\mb a_{i}^{+}\right)+\frac{1}{8}\E_{\mc D}\left(\left|\langle\mb a_{i}^{-},\mb x_{\star}\rangle\right|^{2}\mb a_{i}^{-}\otimes\mb a_{i}^{-}\right)\\
 & \phantom{=}-\frac{1}{8}\E_{\mc D}\left(\left|\langle\mb a_{i}^{+},\mb x_{\star}^{-}\rangle\right|^{2}\mb a_{i}^{+}\otimes\mb a_{i}^{+}\right)-\frac{1}{8}\E_{\mc D}\left(\left|\langle\mb a_{i}^{-},\mb x_{\star}^{-}\rangle\right|^{2}\mb a_{i}^{-}\otimes\mb a_{i}^{-}\right)\\
 & =\frac{1}{4}\left(2\mb x_{\star}\otimes\mb x_{\star}+\left\lVert \mb x_{\star}\right\rVert ^{2}\mb I\right)-\frac{1}{4}\left(2\mb x_{\star}^{-}\otimes\mb x_{\star}^{-}+\left\lVert \mb x_{\star}^{-}\right\rVert ^{2}\mb I\right)\\
 & =\frac{1}{2}\left(\mb x_{\star}\otimes\mb x_{\star}-\mb x_{\star}^{-}\otimes\mb x_{\star}^{-}\right)\,.
\end{align*}
 By the triangle inequality we can write
\begin{align*}
\left\lVert \mb S_{n}-\E_{\mc D}\mb S_{n}\right\rVert _{\mr{op}} & \le\frac{1}{8}\left\lVert \frac{1}{n}\sum_{i=1}^{n}\left|\langle\mb a_{i}^{+},\mb x_{\star}\rangle\right|^{2}\mb a_{i}^{+}\otimes\mb a_{i}^{+}-2\mb x_{\star}\otimes\mb x_{\star}-\left\lVert \mb x_{\star}\right\rVert ^{2}\mb I\right\rVert _{\mr{op}}\\
 & \phantom{\le}+\frac{1}{8}\left\lVert \frac{1}{n}\sum_{i=1}^{n}\left|\langle\mb a_{i}^{-},\mb x_{\star}\rangle\right|^{2}\mb a_{i}^{-}\otimes\mb a_{i}^{-}-2\mb x_{\star}\otimes\mb x_{\star}-\left\lVert \mb x_{\star}\right\rVert ^{2}\mb I\right\rVert _{\mr{op}}\\
 & \phantom{\le}+\frac{1}{8}\left\lVert \frac{1}{n}\sum_{i=1}^{n}\left|\langle\mb a_{i}^{+},\mb x_{\star}^{-}\rangle\right|^{2}\mb a_{i}^{+}\otimes\mb a_{i}^{+}-2\mb x_{\star}^{-}\otimes\mb x_{\star}^{-}-\left\lVert \mb x_{\star}^{-}\right\rVert ^{2}\mb I\right\rVert _{\mr{op}}\\
 & \phantom{\le}+\frac{1}{8}\left\lVert \frac{1}{n}\sum_{i=1}^{n}\left|\langle\mb a_{i}^{-},\mb x_{\star}^{-}\rangle\right|^{2}\mb a_{i}^{-}\otimes\mb a_{i}^{-}-2\mb x_{\star}^{-}\otimes\mb x_{\star}^{-}-\left\lVert \mb x_{\star}^{-}\right\rVert ^{2}\mb I\right\rVert _{\mr{op}}\,.
\end{align*}
 Each of the summands on the right-hand side is small for a sufficiently
large $n$. For example, as shown in \citep[Lemma 7.4]{Candes2014Phase},
if $n\ge C_{\tau}d\log d$ for a sufficiently large constant $C_{\tau}$
that depends only on $\tau\in\left(0,1\right)$, then 
\begin{align*}
\left\lVert \frac{1}{n}\sum_{i=1}^{n}\left|\langle\mb a_{i}^{+},\mb x_{\star}\rangle\right|^{2}\mb a_{i}^{+}\otimes\mb a_{i}^{+}-2\mb x_{\star}\otimes\mb x_{\star}-\left\lVert \mb x_{\star}\right\rVert ^{2}\mb I\right\rVert _{\mr{op}} & \le\tau\left\lVert \mb x_{\star}\right\rVert ^{2}\,,
\end{align*}
with probability $\ge1-5\exp\left(-4\tau d\right)-4d^{-2}$. Clearly,
we can write similar inequalities for the other three summands and
by a simple union bound conclude that 
\begin{align}
\left\lVert \mb S_{n}-\E_{\mc D}\mb S_{n}\right\rVert _{\mr{op}} & \le\frac{\tau}{4}\left(\left\lVert \mb x_{\star}\right\rVert ^{2}+\left\lVert \mb x_{\star}^{-}\right\rVert ^{2}\right)=\frac{\tau}{2}\left\lVert \mb x_{\star}\right\rVert ^{2}\,,\label{eq:Sn-concentrates}
\end{align}
 holds with probability $\ge1-c_{\tau}d^{-2}$ for some absolute constant
$c_{\tau}$ depending only on $\tau$. Recall that $\E_{\mc D}\mb S_{n}=\left(\mb x_{\star}\otimes\mb x_{\star}-\mb x_{\star}^{-}\otimes\mb x_{\star}^{-}\right)/2$.
Because we chose $\left\lVert \mb x_{\star}^{(1)}\right\rVert =\left\lVert \mb x{}_{\star}^{(2)}\right\rVert $
and by the construction of $\mb x_{\star}$ and $\mb x_{\star}^{-}$
we have $\langle\mb x_{\star},\mb x_{\star}^{-}\rangle=0$. Thus $\mb x_{\star}$
and $\mb x_{\star}^{-}$ are eigenvectors of $\E_{\mc D}\mb S_{n}$.
We may assume that $\langle\mb v_{\max},\mb x_{\star}\rangle\ge0$;
otherwise we can simply use $-\mb x_{\star}$ as the target. Then,
on the event \eqref{eq:Sn-concentrates}, a variant of the Davis-Kahan
theorem \citep[Corollary 3]{Yu2015Useful} ensures
\begin{align*}
\left\lVert \mb v_{\max}-\frac{\mb x_{\star}}{\left\lVert \mb x_{\star}\right\rVert }\right\rVert  & \le\frac{2^{1/2}\tau\left\lVert \mb x_{\star}\right\rVert ^{2}}{\left\lVert \mb x_{\star}\right\rVert ^{2}/2}=2^{3/2}\tau\,.
\end{align*}
 Since $\mb a_{0}$ is defined by \eqref{eq:a0-bilinear}, we equivalently
obtain 
\begin{align*}
\left\lVert \mb a_{0}-\left(\frac{\lambda_{\max}}{2}\right)^{1/2}\frac{\mb x_{\star}}{\left\lVert \mb x_{\star}\right\rVert }\right\rVert  & \le2\tau\lambda_{\max}^{1/2}\,.
\end{align*}
 Using \eqref{eq:Sn-concentrates}, it is also easy to show that 
\[
\frac{1-\tau}{2}\left\lVert \mb x_{\star}\right\rVert ^{2}\le\lambda_{\max}\le\frac{1+\tau}{2}\left\lVert \mb x_{\star}\right\rVert ^{2}\,.
\]
 Therefore, we deduce that 
\begin{align}
\left\lVert \mb a_{0}-\frac{1}{2}\mb x_{\star}\right\rVert  & \le\left\lVert \mb a_{0}-\left(\frac{\lambda_{\max}}{2}\right)^{1/2}\frac{\mb x_{\star}}{\left\lVert \mb x_{\star}\right\rVert }\right\rVert +\left\lVert \left(\frac{\lambda_{\max}}{2}\right)^{1/2}\frac{\mb x_{\star}}{\left\lVert \mb x_{\star}\right\rVert }-\frac{1}{2}\mb x_{\star}\right\rVert \nonumber \\
 & \le2^{1/2}\tau\left(1+\frac{\tau}{2}\right)\left\lVert \mb x_{\star}\right\rVert +\frac{1}{2}\tau\left\lVert \mb x_{\star}\right\rVert \,.\label{eq:a0-xstar}
\end{align}
It follows from Lemma \ref{lem:gradient-concentration} for $\mb x=\mb x_{\star}$,
\eqref{eq:a0-xstar}, and \eqref{eq:lambda-Lambda-BL}, that if $n\overset{\tau}{\gtrsim}\left(d+\log\frac{4}{\delta}\right)\log d$,
then
\begin{align*}
\left\lVert \frac{1}{2n}\sum_{i=1}^{n}\nabla f_{i}^{+}(\mb x_{\star})+\nabla f_{i}^{-}(\mb x_{\star})-\mb a_{0}\right\rVert  & \le C_{\tau}\lambda_{\mc D}\,,
\end{align*}
 with probability $\ge1-c_{\tau}d^{-2}$. Choosing an appropriate
value for $\tau$ in terms of $\varepsilon$, the constant $C_{\tau}$
can also be made smaller than $(1-\varepsilon)/2$, thereby guaranteeing
\eqref{eq:ApproxOracle}.

\subsection{Numerical Experiments}

To evaluate the proposed method numerically, we ran $100$ trials
with the standard Gaussian measurements for each pair of $d_{1}=d_{2}=d/2\in\left\{ 50,100,150\right\} $
and $n/d\in\left\{ 5,6,7,8,9\right\} $. The signal pairs $\mb x_{\star}^{(1)}$
and $\mb x_{\star}^{(2)}$ are drawn independently and uniformly from the $d/2$-dimensional
unit sphere in each trial. We solved an equivalent form of \eqref{eq:AnchoredERM}
which is the quadratically-constrained linear maximization 
\begin{equation}
\begin{aligned}\max_{\mb x\in \mbb{R}^{d_1+d_2},\mb w\in\mbb{R}^n}\  & \langle\mb a_{0},\mb x\rangle-\frac{1}{n}\langle\mb 1_{n},\mb w\rangle\\
\text{subject to}\  & \frac{1}{4}\left|\langle\mb a_{i}^{+},\mb x\rangle\right|^{2}-y_{i}\le w_{i}, & i=1,\dotsc,n\\
 & \frac{1}{4}\left|\langle\mb a_{i}^{-},\mb x\rangle\right|^{2}\le w_{i}, & i=1,\dotsc,n\,,
\end{aligned}
\label{eq:QCLM}
\end{equation}
 where $\mb 1_{n}$ denotes the $n$-dimensional all-one vector, using
the Gurobi solver \citep{Gurobi} through the CVX package \citep{CVX}.
This solver relies on an \emph{interior point method} for solving
the second order cone program (SOCP) corresponding to \eqref{eq:QCLM}.
For better scalability, first order methods including stochastic and
incremental methods can be used to solve \eqref{eq:AnchoredERM} directly.
We did not intend in this paper to find the best convex optimization
method for solving \eqref{eq:AnchoredERM}. 

Figure \ref{fig:RelErr} shows the median of the relative
error computed as 
\[
\text{Relative Error}=\sqrt{\frac{\Vert\sqrt{\Vert\widehat{\mb x}^{(2)}\Vert/\Vert\widehat{\mb x}^{(1)}\Vert}\,\widehat{\mb x}^{(1)}-\mb x_{\star}^{(1)}\Vert^{2}+\Vert\sqrt{\Vert\widehat{\mb x}^{(1)}\Vert/\Vert\widehat{\mb x}^{(2)}\Vert}\,\widehat{\mb x}^{(2)}-\mb x_{\star}^{(2)}\Vert^{2}}{\Vert\mb x_{\star}^{(1)}\Vert^{2}+\Vert\mb x{}_{\star}^{(2)}\Vert^{2}}}\,.
\]
 The experiment suggests that the proposed method succeeds when the
oversampling ratio is around eight (i.e., $n\approx8(d_{1}+d_{2})=8d$).

\begin{figure}
\noindent
\centering

\includegraphics[width=1\textwidth]{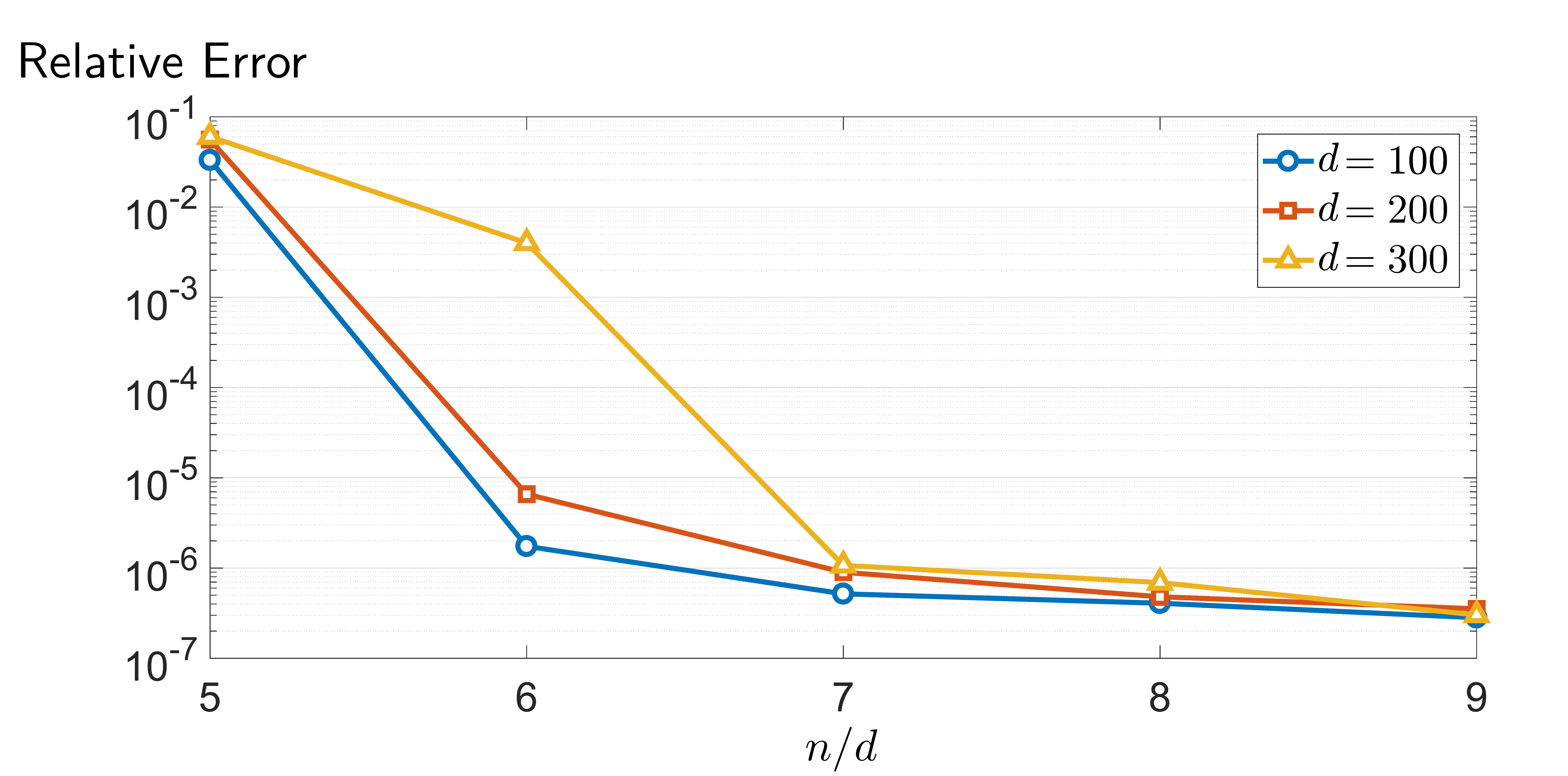}\caption{\label{fig:RelErr}Relative error of the estimate \eqref{eq:AnchoredERM}
versus the oversampling ratio $n/d$}

\end{figure}

\appendix

\section{Technical Lemmas}

\subsection{Lemmas used in Section \ref{sec:main-result}}
\begin{lem}
\label{lem:trucated2empirical} For any $\alpha>0$, with probability
$\ge1-\delta/2$ we have
\begin{align*}
\frac{1}{n}\sum_{i=1}^{n}\frac{1}{\alpha}\E_{\mb z\sim\gamma_{\mb h}}\left(\left[\alpha q_{i}(\mb z)\right]_{\le1}\right) & \le\frac{1}{n}\sum_{i=1}^{n}q_{i}(\mb h)+\sigma\varGamma_{\mc D}+\frac{1}{\alpha}\sqrt{\frac{\log\frac{2}{\delta}}{2n}}\,,
\end{align*}
for all $\mb h$. 
\end{lem}
\begin{proof}
The triangle inequality and subadditivity of $u\mapsto\left[u\right]_{\le1}$
over the nonnegative real numbers yields
\begin{align*}
& \frac{1}{n}\sum_{i=1}^{n}\frac{1}{\alpha}\E_{\mb z\sim\gamma_{\mb h}}\left(\left[\alpha q_{i}(\mb z)\right]_{\le1}\right)\\
 & \le\frac{1}{n}\sum_{i=1}^{n}\frac{1}{\alpha}\E_{\mb z\sim\gamma_{\mb h}}\left(\left[\alpha q_{i}(\mb h)+\left|\alpha q_{i}(\mb z)-\alpha q_{i}(\mb h)\right|\right]_{\le1}\right)\\
 & \le\frac{1}{n}\sum_{i=1}^{n}\frac{1}{\alpha}\E_{\mb z\sim\gamma_{\mb h}}\left(\left[\alpha q_{i}\left(\mb h\right)\right]_{\le1}\right)+\frac{1}{n}\sum_{i=1}^{n}\frac{1}{\alpha}\E_{\mb z\sim\gamma_{\mb h}}\left(\left[\left|\alpha q_{i}(\mb z)-\alpha q_{i}(\mb h)\right|\right]_{\le1}\right)\,.
\end{align*}
Clearly, $\left[\alpha q_{i}(\mb h)\right]_{\le1}\le\alpha q_{i}(\mb h)$.
Thus, we only need to bound the second term in the above inequality.
Using \eqref{eq:triangle-like} followed by the Hoeffding's inequality shows that
\begin{align*}
  & \frac{1}{n}\sum_{i=1}^{n}\E_{\mb z\sim\gamma_{\mb h}}\left(\left[\alpha\left|q_{i}(\mb z)-q_{i}(\mb h)\right|\right]_{\le1}\right) \\
  & \le \frac{1}{n}\sum_{i=1}^{n}\E_{\mb z\sim\gamma_{\mb h}}\left(\left[\alpha\left|q_{i}(\mb z -\mb{h})\right|\right]_{\le1}\right) \\
  & = \frac{1}{n}\sum_{i=1}^{n}\E_{\mb z\sim\gamma_{\mb 0}}\left(\left[\alpha\left|q_{i}(\mb z)\right|\right]_{\le1}\right) \\
  & \le\E_{\mc D}\E_{\mb z\sim\gamma_{\mb 0}}\left(\left[\alpha\left|q_{i}(\mb z)\right|\right]_{\le1}\right)+\sqrt{\frac{\log\frac{2}{\delta}}{2n}}\,, 
\end{align*}
holds with probability $\ge 1-\delta/2$ for all $\mb{h}$. Therefore, on this event we have
\begin{align*}
& \frac{1}{n}\sum_{i=1}^{n}\frac{1}{\alpha}\E_{\mb z\sim\gamma_{\mb h}}\left(\left[\alpha q_{i}(\mb z)\right]_{\le1}\right)\\
 & \le\frac{1}{n}\sum_{i=1}^{n}q_{i}(\mb h)+\frac{1}{\alpha}\left[\E_{\mc D}\E_{\mb z\sim\gamma_{\mb 0}}\left(\alpha q_{i}(\mb z)\right)\right]_{\le1}+\frac{1}{\alpha}\sqrt{\frac{\log\frac{2}{\delta}}{2n}}\\
 & \le\frac{1}{n}\sum_{i=1}^{n}q_{i}(\mb h)+\sigma\varGamma_{\mc D}+\frac{1}{\alpha}\sqrt{\frac{\log\frac{2}{\delta}}{2n}}
\end{align*}
where concavity of $u\mapsto\left[u\right]_{\le1}$ is used in the first inequality, and the second
inequality follows from the fact that $\left[u\right]_{\le1}\le u$,
the Cauchy-Schwarz inequality, and the definition \eqref{eq:sensitivity}.
\end{proof}
\begin{lem}
\label{lem:variance-term}For all $\mb h$, we have
\begin{align*}
\sqrt{\E_{\mb z\sim\gamma_{\mb h}}\E_{\mc D}q_{1}^{2}(\mb z)} & \le\eta_{\mc D}\E_{\mc D}\left(q_{1}(\mb h)\right)+\sigma\varGamma_{\mc D}\,.
\end{align*}
\end{lem}
\begin{proof}
It immediately follows from the triangle inequality, \eqref{eq:triangle-like},
and the equivalence of $\mb z\sim\gamma_{\mb h}$ and $\mb z-\mb h\sim\gamma_{\mb 0}$,
that
\begin{align*}
\sqrt{\E_{\mb z\sim\gamma_{\mb h}}\E_{\mc D}q_{1}^{2}(\mb z)} & \le\sqrt{\E_{\mc D}\left(q_{1}^{2}(\mb h)\right)}+\sqrt{\E_{\mb z\sim\gamma_{\mb h}}\E_{\mc D}\left(\left(q_{1}(\mb z)-q_{1}(\mb h)\right)^{2}\right)}\\
 & \le\sqrt{\E_{\mc D}\left(q_{1}^{2}(\mb h)\right)}+\sqrt{\E_{\mb z\sim\gamma_{\mb h}}\E_{\mc D}\left(q_{1}^{2}(\mb z-\mb h)\right)}\\
 & \le\sqrt{\E_{\mc D}\left(q_{1}^{2}(\mb h)\right)}+\sqrt{\E_{\mb z\sim\gamma_{0}}\E_{\mc D}\left(q_{1}^{2}(\mb z)\right)}\,,
\end{align*}
which by the assumption \eqref{eq:tail-weight} and definition \eqref{eq:sensitivity}
yields is the desired bound.
\end{proof}

\subsection{Lemmas used in Section \ref{sec:bilinear-regression}}

Note that in the following lemmas the functions $f_{i}^{+}$ and $f_{i}^{-}$
are defined as in \eqref{eq:quads}.
\begin{lem}
\label{lem:conversion-lemma}With $\mc K_{t,\theta}$ defined by \eqref{eq:Kt,theta},
suppose that 
\begin{align*}
\frac{1}{n}\sum_{i=1}^{n}\left|\langle\nabla f_{i}^{+}(\mb x_{\star})-\nabla f_{i}^{-}(\mb x_{\star}),\mb h\rangle\right| & \ge(1-\varepsilon)\E_{\mc D}\left(\left|\langle\nabla f^{+}(\mb x_{\star})-\nabla f^{-}(\mb x_{\star}),\mb h\rangle\right|\right)\,,
\end{align*}
 for all vectors $\mb h\in\mc K_{1,\frac{1}{2}}$ . Then, for all
$t\in\mbb R\backslash\{0\}$ with $\sqrt{2/3}\le\left|t\right|\le\sqrt{3/2}$,
and all vectors $\mb h\in\mc K_{t,0}$, we have 
\begin{align*}
& \frac{1}{n}\sum_{i=1}^{n}\left|\langle\nabla f_{i}^{+}(D_{t}(\mb x_{\star}))-\nabla f_{i}^{-}(D_{t}(\mb x_{\star})),\mb h\rangle\right| \\& \ge\left(1-\varepsilon\right)\E_{\mc D}\left(\left|\langle\nabla f^{+}(\mb x_{\star})-\nabla f^{-}(\mb x_{\star}),D_{t^{-1}}(\mb h)\rangle\right|\right)\\
 & =\left(1-\varepsilon\right)\E_{\mc D}\left(\left|\langle\nabla f^{+}(D_{t}(\mb x_{\star}))-\nabla f^{-}(D_{t}(\mb x_{\star})),\mb h\rangle\right|\right)\,.
\end{align*}
\end{lem}
\begin{proof}
We have the identity
\begin{align}
\langle\nabla f_{i}^{+}(D_{t}(\mb x_{\star}))-\nabla f_{i}^{-}(D_{t}(\mb x_{\star})),\mb h\rangle & =\langle\mb a_{i}^{(1)}\otimes\mb a_{i}^{(2)},t\mb x_{\star}^{(1)}\otimes\mb h^{(2)}+\mb h^{(1)}\otimes t^{-1}\mb x_{\star}^{(2)}\rangle\nonumber \\
 & =\langle\nabla f_{i}^{+}(\mb x_{\star})-\nabla f_{i}^{-}(\mb x_{\star}),D_{t^{-1}}\left(\mb h\right)\rangle\,,\label{eq:bilinear-identity}
\end{align}
for every $\mb h$ and $t\in\mbb R\backslash\{0\}$. Furthermore,
because $\mb h\in\mc K_{t,0}$, by definition $\langle D_{t}(\mb x_{\star}^{-}),\mb h\rangle=0$,
thereby we have the following 
\begin{align*}
\left|\langle D_{t^{-1}}(\mb x_{\star}^{-}),\mb h\rangle\right| & =\left|\langle D_{t^{-1}}(\mb x_{\star}^{-})-D_{t}(\mb x_{\star}^{-}),\mb h\rangle\right|\\
 & =\left|t-t^{-1}\right|\left|\langle\mb x_{\star},\mb h\rangle\right|\\
\\
 & \le\max\left\{ \left|t^{2}-1\right|,\left|t^{-2}-1\right|\right\} \left\lVert \mb x_{\star}\right\rVert \left\lVert D_{t^{-1}}(\mb h)\right\rVert \,.
\end{align*}
Applying the bound $\sqrt{2/3}\le\left|t\right|\le\sqrt{3/2}$, we
obtain 
\begin{align*}
\left|\langle\mb x_{\star}^{-},D_{t^{-1}}(\mb h)\rangle\right|=\left|\langle D_{t^{-1}}(\mb x_{\star}^{-}),\mb h\rangle\right| & \le\frac{1}{2}\left\lVert D_{t^{-1}}(\mb x_{\star})\right\rVert \left\lVert \mb h\right\rVert \,,
\end{align*}
 which means that $D_{t^{-1}}(\mb h)\in\mc K_{1,\frac{1}{2}}$. Therefore,
it follows from the assumption of the lemma that
\begin{align*}
& \frac{1}{n}\sum_{i=1}^{n}\left|\langle\nabla f_{i}^{+}(\mb x_{\star})-\nabla f_{i}^{-}(\mb x_{\star}),D_{t^{-1}}\left(\mb h\right)\rangle\right|\\
 & \ge(1-\varepsilon)\E_{\mc D}\left(\left|\langle\nabla f^{+}(\mb x_{\star})-\nabla f^{-}(\mb x_{\star}),D_{t^{-1}}\left(\mb h\right)\rangle\right|\right)\,,
\end{align*}
which, using \eqref{eq:bilinear-identity}, implies
\begin{align*}
& \frac{1}{n}\sum_{i=1}^{n}\left|\langle\nabla f_{i}^{+}(D_{t}(\mb x_{\star}))-\nabla f_{i}^{-}(D_{t}(\mb x_{\star})),\mb h\rangle\right| \\& \ge\left(1-\varepsilon\right)\E_{\mc D}\left(\left|\langle\nabla f^{+}(D_{t}(\mb x_{\star}))-\nabla f^{-}(D_{t}(\mb x_{\star})),\mb h\rangle\right|\right)\,,
\end{align*}
as desired.
\end{proof}
We use standard matrix concentration inequalities to establish Lemmas
\ref{lem:gradient-concentration} and \ref{lem:xstar-xhatstar} below.
We can upper bound $\left\lVert \frac{1}{n}\sum_{i=1}^{n}\mb a_{i}^{+}\otimes\mb a_{i}^{+}-\mb I\right\rVert _{\mr{op}}$
by a standard covering argument as in \citep[Theorem 5.39]{Vershynin2012Introduction}
which guarantees 
\begin{align}
\left\lVert \frac{1}{n}\sum_{i=1}^{n}\mb a_{i}^{+}\otimes\mb a_{i}^{+}-\mb I\right\rVert _{\mr{op}} & \le C\max\left\{ \sqrt{\frac{d}{n}}+\sqrt{\frac{\log\frac{4}{\delta}}{n}},\left(\sqrt{\frac{d}{n}}+\sqrt{\frac{\log\frac{4}{\delta}}{n}}\right)^{2}\right\} \,,\label{eq:matrix-concentration-a+}
\end{align}
 with probability $\ge1-\delta/2$ for a sufficiently large absolute
constant $C>0$. Similarly, we have 
\begin{align}
\left\lVert \frac{1}{n}\sum_{i=1}^{n}\mb a_{i}^{-}\otimes\mb a_{i}^{-}-\mb I\right\rVert _{\mr{op}} & \le C\max\left\{ \sqrt{\frac{d}{n}}+\sqrt{\frac{\log\frac{4}{\delta}}{n}},\left(\sqrt{\frac{d}{n}}+\sqrt{\frac{\log\frac{4}{\delta}}{n}}\right)^{2}\right\} \,,\label{eq:matrix-concentration-a-}
\end{align}
with probability $\ge1-\delta/2$.

The first lemma below is an immediate consequence of the matrix concentration
inequalities above and is stated merely for reference.
\begin{lem}
\label{lem:gradient-concentration} On the event that \eqref{eq:matrix-concentration-a+}
and \eqref{eq:matrix-concentration-a-} hold, we have
\begin{equation}
\begin{aligned} & \left\lVert \frac{1}{2n}\sum_{i=1}^{n}\nabla f_{i}^{+}(\mb x)+\nabla f_{i}^{-}(\mb x)-\frac{1}{2}\mb x\right\rVert \\
 & \le C\max\left\{ \sqrt{\frac{d}{n}}+\sqrt{\frac{\log\frac{4}{\delta}}{n}},\left(\sqrt{\frac{d}{n}}+\sqrt{\frac{\log\frac{4}{\delta}}{n}}\right)^{2}\right\} \left\lVert \mb x\right\rVert \,,
\end{aligned}
\label{eq:gradient-concentration}
\end{equation}
 for every $\mb x\in\mbb R^{d}$.
\end{lem}
\begin{proof}
By definition
\begin{align*}
\nabla f_{i}^{+}(\mb x)+\nabla f_{i}^{-}(\mb x) & =\frac{1}{2}\left(\mb a_{i}^{+}\otimes\mb a_{i}^{+}\right)\mb x+\frac{1}{2}\left(\mb a_{i}^{-}\otimes\mb a_{i}^{-}\right)\mb x\,.
\end{align*}
for any $\mb x$. A simple application of triangle inequality yields
\begin{gather*}
\begin{aligned} & \left\lVert \frac{1}{2n}\sum_{i=1}^{n}\nabla f_{i}^{+}(\mb x)+\nabla f_{i}^{-}(\mb x)-\frac{1}{2}\mb x\right\rVert \\
 & \le\frac{1}{4}\left\lVert \frac{1}{n}\sum_{i=1}^{n}\mb a_{i}^{+}\otimes\mb a_{i}^{+}-\mb I\right\rVert _{\mr{op}}\left\lVert \mb x\right\rVert +\frac{1}{4}\left\lVert \frac{1}{n}\sum_{i=1}^{n}\mb a_{i}^{-}\otimes\mb a_{i}^{-}-\mb I\right\rVert _{\mr{op}}\left\lVert \mb x\right\rVert \,.
\end{aligned}
\end{gather*}
 The result follows immediately using the matrix concentration inequalities
\eqref{eq:matrix-concentration-a+} and \eqref{eq:matrix-concentration-a-}. 
\end{proof}
\begin{lem}
\label{lem:xstar-xhatstar}There exists an absolute constant $C>0$
such that
\begin{equation}
\begin{aligned} & \left\lVert \mb a_{0}-\frac{1}{2n}\sum_{i=1}^{n}\nabla f_{i}^{+}(\mb x_{\star})+\nabla f_{i}^{-}(\mb x_{\star})\right\rVert \\
 & \ge\frac{1}{8}\left(1-C\max\left\{ \sqrt{\frac{d}{n}}+\sqrt{\frac{\log\frac{4}{\delta}}{n}},\left(\sqrt{\frac{d}{n}}+\sqrt{\frac{\log\frac{4}{\delta}}{n}}\right)^{2}\right\} \right)\left\lVert \widehat{\mb x}_{\star}-\mb x_{\star}\right\rVert \,,
\end{aligned}
\label{eq:xhatstar-xstar}
\end{equation}
holds with probability $\ge1-\delta$. Furthermore, for a sufficiently
large $n$, on the same event we have 
\begin{equation}
\begin{aligned} & \left\lVert \mb a_{0}-\frac{1}{2n}\sum_{i=1}^{n}\nabla f_{i}^{+}(\widehat{\mb x}_{\star})+\nabla f_{i}^{-}(\widehat{\mb x}_{\star})\right\rVert \\
 & \le\frac{5+3C\max\left\{ \sqrt{\frac{d}{n}}+\sqrt{\frac{\log\frac{4}{\delta}}{n}},\left(\sqrt{\frac{d}{n}}+\sqrt{\frac{\log\frac{4}{\delta}}{n}}\right)^{2}\right\} }{1-C\max\left\{ \sqrt{\frac{d}{n}}+\sqrt{\frac{\log\frac{4}{\delta}}{n}},\left(\sqrt{\frac{d}{n}}+\sqrt{\frac{\log\frac{4}{\delta}}{n}}\right)^{2}\right\} }\left\lVert \mb a_{0}-\frac{1}{2n}\sum_{i=1}^{n}\nabla f_{i}^{+}(\mb x_{\star})+\nabla f_{i}^{-}(\mb x_{\star})\right\rVert \,.
\end{aligned}
\label{eq:a0-nabla-xhatstar}
\end{equation}
\end{lem}
\begin{proof}
First we prove \eqref{eq:xhatstar-xstar}. By optimality of $\widehat{\mb x}$
in \eqref{eq:AnchoredERM}, we can write 
\begin{align*}
\langle\mb a_{0},\widehat{\mb x}-\mb x_{\star}\rangle & \ge\frac{1}{n}\sum_{i=1}^{n}\max\left(f_{i}^{+}(\widehat{\mb x})-f_{i}^{+}(\mb x_{\star}),f_{i}^{-}(\widehat{\mb x})-f_{i}^{-}(\mb x_{\star})\right)\\
 & \ge\frac{1}{2n}\sum_{i=1}^{n}f_{i}^{+}(\widehat{\mb x})+f_{i}^{-}(\widehat{\mb x})-f_{i}^{+}(\mb x_{\star})-f_{i}^{-}(\mb x_{\star}).
\end{align*}
Then, subtracting $\langle\frac{1}{2n}\sum_{i=1}^{n}\nabla f_{i}^{+}(\mb x_{\star})+\nabla f_{i}^{-}(\mb x_{\star}),\widehat{\mb x}-\mb x_{\star}\rangle$
yields 
\begin{align}
 & \langle\mb a_{0}-\frac{1}{2n}\sum_{i=1}^{n}\nabla f_{i}^{+}(\mb x_{\star})+\nabla f_{i}^{-}(\mb x_{\star}),\widehat{\mb x}-\mb x_{\star}\rangle\nonumber \\
 & \ge\frac{1}{2n}\sum_{i=1}^{n}f_{i}^{+}(\widehat{\mb x})+f_{i}^{-}(\widehat{\mb x})-f_{i}^{+}(\mb x_{\star})-f_{i}^{-}(\mb x_{\star})-\langle\nabla f_{i}^{+}(\mb x_{\star})+\nabla f_{i}^{-}(\mb x_{\star}),\widehat{\mb x}-\mb x_{\star}\rangle\nonumber \\
 & =\frac{1}{2n}\sum_{i=1}^{n}\frac{1}{4}\left|\langle\mb a_{i}^{+},\widehat{\mb x}-\mb x_{\star}\rangle\right|^{2}+\frac{1}{4}\left|\langle\mb a_{i}^{-},\widehat{\mb x}-\mb x_{\star}\rangle\right|^{2}\,.\label{eq:a0-xhat-xstar}
\end{align}
 Applying the Cauchy-Schwarz inequality to the first line, and the
standard matrix concentration inequalities \eqref{eq:matrix-concentration-a+}
and \eqref{eq:matrix-concentration-a-} in the third line, we obtain
with probability $\ge1-\delta$ that 
\begin{align*}
 & \left\lVert \mb a_{0}-\frac{1}{2n}\sum_{i=1}^{n}\nabla f_{i}^{+}(\mb x_{\star})+\nabla f_{i}^{-}(\mb x_{\star})\right\rVert \left\lVert \widehat{\mb x}-\mb x_{\star}\right\rVert \\
 & \ge\frac{1}{4}\left(1-C\max\left\{ \sqrt{\frac{d}{n}}+\sqrt{\frac{\log\frac{4}{\delta}}{n}},\left(\sqrt{\frac{d}{n}}+\sqrt{\frac{\log\frac{4}{\delta}}{n}}\right)^{2}\right\} \right)\left\lVert \widehat{\mb x}-\mb x_{\star}\right\rVert ^{2}\,,
\end{align*}
and thereby 
\[
\begin{aligned} & \left\lVert \mb a_{0}-\frac{1}{2n}\sum_{i=1}^{n}\nabla f_{i}^{+}(\mb x_{\star})+\nabla f_{i}^{-}(\mb x_{\star})\right\rVert \\
 & \ge\frac{1}{4}\left(1-C\max\left\{ \sqrt{\frac{d}{n}}+\sqrt{\frac{\log\frac{4}{\delta}}{n}},\left(\sqrt{\frac{d}{n}}+\sqrt{\frac{\log\frac{4}{\delta}}{n}}\right)^{2}\right\} \right)\left\lVert \widehat{\mb x}-\mb x_{\star}\right\rVert \,.
\end{aligned}
\]
Finally, it follows from the triangle inequality and the definition
of $\widehat{\mb x}_{\star}$ in \eqref{eq:closest-equivalent} that
\[
\left\lVert \widehat{\mb x}_{\star}-\mb x_{\star}\right\rVert \le\left\lVert \widehat{\mb x}-\widehat{\mb x}_{\star}\right\rVert +\left\lVert \widehat{\mb x}-\mb x_{\star}\right\rVert \le2\left\lVert \widehat{\mb x}-\mb x_{\star}\right\rVert \,,
\]
 which together with the previous bound guarantees \eqref{eq:xhatstar-xstar}.

Next, we prove \eqref{eq:a0-nabla-xhatstar}. By the triangle inequality
and the fact that $\nabla f_{i}^{\pm}(\mb x)$ is linear in $\mb x$,
we have
\allowdisplaybreaks
\begin{align*}
&\left\lVert \mb a_{0}-\frac{1}{2n}\sum_{i=1}^{n}\nabla f_{i}^{+}(\widehat{\mb x}_{\star})+\nabla f_{i}^{-}(\widehat{\mb x}_{\star})\right\rVert\\
  & \le\left\lVert \frac{1}{2n}\sum_{i=1}^{n}\nabla f_{i}^{+}(\mb x_{\star})+\nabla f_{i}^{-}(\mb x_{\star})-\nabla f_{i}^{+}(\widehat{\mb x}_{\star})-\nabla f_{i}^{-}(\widehat{\mb x}_{\star})\right\rVert \\
 & \phantom{\le}+\left\lVert \mb a_{0}-\frac{1}{2n}\sum_{i=1}^{n}\nabla f_{i}^{+}(\mb x_{\star})+\nabla f_{i}^{-}(\mb x_{\star})\right\rVert \\
 & =\left\lVert \frac{1}{2n}\sum_{i=1}^{n}\nabla f_{i}^{+}(\mb x_{\star}-\widehat{\mb x}_{\star})+\nabla f_{i}^{-}(\mb x_{\star}-\widehat{\mb x}_{\star})\right\rVert \\
 & \phantom{\le}+\left\lVert \mb a_{0}-\frac{1}{2n}\sum_{i=1}^{n}\nabla f_{i}^{+}(\mb x_{\star})+\nabla f_{i}^{-}(\mb x_{\star})\right\rVert \,.
\end{align*}
Recall, from the first part of the proof, that \eqref{eq:matrix-concentration-a+}
and \eqref{eq:matrix-concentration-a-} hold with probability $\ge1-\delta$.
Then, on the same event, Lemma \ref{lem:gradient-concentration} implies
that 
\begin{align*}
 & \left\lVert \mb a_{0}-\frac{1}{2n}\sum_{i=1}^{n}\nabla f_{i}^{+}(\widehat{\mb x}_{\star})+\nabla f_{i}^{-}(\widehat{\mb x}_{\star})\right\rVert \\
 & \le\frac{1}{2}\left(1+C\max\left\{ \sqrt{\frac{d}{n}}+\sqrt{\frac{\log\frac{4}{\delta}}{n}},\left(\sqrt{\frac{d}{n}}+\sqrt{\frac{\log\frac{4}{\delta}}{n}}\right)^{2}\right\} \right)\left\lVert \widehat{\mb x}_{\star}-\mb x_{\star}\right\rVert \\
 & \phantom{\le}+\left\lVert \mb a_{0}-\frac{1}{2n}\sum_{i=1}^{n}\nabla f_{i}^{+}(\mb x_{\star})+\nabla f_{i}^{-}(\mb x_{\star})\right\rVert \,.
\end{align*}
 Therefore, if $n$ is sufficiently large to ensure the right-hand
side of \eqref{eq:xhatstar-xstar} is nonnegative, we deduce that
\eqref{eq:a0-nabla-xhatstar} holds as well. 
\end{proof}
\begin{lem}
\label{lem:gaussian-norm}For any matrix $\mb A$ and standard normal
random vector $\mb z$ (of appropriate dimension) we have 
\begin{align*}
\E\left\lVert \mb A\mb z\right\rVert  & \ge\sqrt{\frac{2}{\pi}}\left\lVert \mb A\right\rVert _{\F}\,.
\end{align*}
\end{lem}
\begin{proof}
The Euclidean and Frobenius norms as well as the standard normal distribution
are rotationally invariant. Thus, the claim can be reduced to the
case where $\mb A$ is diagonal with nonzero diagonal entries $s_{1},s_{2},\dotsc,s_{r}$
and $\left\lVert \mb A\mb z\right\rVert =\sqrt{\sum_{i=1}^{r}s_{i}^{2}z_{i}^{2}}$.
By concavity of $u\mapsto\sqrt{u}$ and Jensen's inequality we have
\begin{align*}
\sqrt{\left(\sum_{i=1}^{r}s_{i}^{2}\right)^{-1}\sum_{i=1}^{r}s_{i}^{2}z_{i}^{2}} & \ge\left(\sum_{i=1}^{r}s_{i}^{2}\right)^{-1}\sum_{i=1}^{r}s_{i}^{2}\left|z_{i}\right|\,.
\end{align*}
Therefore, taking expectation with respect to $\mb z$ we can conclude
\[
\E\,\sqrt{\sum_{i=1}^{r}s_{i}^{2}z_{i}^{2}}\ge\sqrt{\left(\sum_{i=1}^{r}s_{i}^{2}\right)^{-1}}\sum_{i=1}^{r}s_{i}^{2}\E\left|z_{i}\right|=\sqrt{\frac{2}{\pi}\sum_{i=1}^{r}s_{i}^{2}}=\sqrt{\frac{2}{\pi}}\left\lVert \mb A\right\rVert _{\F}\,.
\]
\end{proof}
\subsection*{Acknowledgements}
This work was supported in part by Semiconductor Research Corporation (SRC).

\bibliographystyle{abbrvnat}
\bibliography{references}

\begin{thebibliography}{50}
\providecommand{\natexlab}[1]{#1}
\providecommand{\url}[1]{\texttt{#1}}
\expandafter\ifx\csname urlstyle\endcsname\relax
  \providecommand{\doi}[1]{doi: #1}\else
  \providecommand{\doi}{doi: \begingroup \urlstyle{rm}\Url}\fi

\bibitem[Aghasi et~al.(2017)Aghasi, Ahmed, and Hand]{Aghasi2017BranchHull}
A.~Aghasi, A.~Ahmed, and P.~Hand.
\newblock Branchhull: Convex bilinear inversion from the entrywise product of
  signals with known signs.
\newblock preprint
  \href{https://arxiv.org/abs/1702.04342}{\texttt{arXiv:1702.04342 [cs.IT]}},
  2017.

\bibitem[Ahmadi and Hall(2017)]{Ahmadi2017DC}
A.~A. Ahmadi and G.~Hall.
\newblock {DC} decomposition of nonconvex polynomials with algebraic
  techniques.
\newblock \emph{Mathematical Programming - Series B}, 2017.

\bibitem[Ahmed et~al.(2014)Ahmed, Recht, and Romberg]{Ahmed2014Blind}
A.~Ahmed, B.~Recht, and J.~Romberg.
\newblock Blind deconvolution using convex programming.
\newblock \emph{{IEEE} Transactions on Information Theory}, 60\penalty0
  (3):\penalty0 1711--1732, March 2014.

\bibitem[Alquier(2008)]{Alquier2008PAC-Bayesian}
P.~Alquier.
\newblock {PAC-Bayesian} bounds for randomized empirical risk minimizers.
\newblock \emph{Mathematical Methods of Statistics}, 17\penalty0 (4):\penalty0
  279--304, Dec 2008.

\bibitem[Alquier and Lounici(2011)]{Alquier2011PAC-Bayesian}
P.~Alquier and K.~Lounici.
\newblock {PAC-Bayesian} bounds for sparse regression estimation with
  exponential weights.
\newblock \emph{Electronic Journal of Statistics}, 5:\penalty0 127--145, 2011.

\bibitem[Audibert and Catoni(2011)]{Audibert2011Robust}
J.-Y. Audibert and O.~Catoni.
\newblock Robust linear least squares regression.
\newblock \emph{the Annals of Statistics}, 39\penalty0 (5):\penalty0
  2766--2794, 10 2011.

\bibitem[Bahmani and Romberg(2015)]{Bahmani2015Lifting}
S.~Bahmani and J.~Romberg.
\newblock Lifting for blind deconvolution in random mask imaging:
  Identifiability and convex relaxation.
\newblock \emph{{SIAM} Journal on Imaging Sciences}, 8\penalty0 (4):\penalty0
  2203--2238, 2015.

\bibitem[Bahmani and Romberg(2017{\natexlab{a}})]{Bahmani2016Phase}
S.~Bahmani and J.~Romberg.
\newblock Phase retrieval meets statistical learning theory: A flexible convex
  relaxation.
\newblock In A.~Singh and J.~Zhu, editors, \emph{Proceedings of the 20th
  International Conference on Artificial Intelligence and Statistics},
  volume~54 of \emph{Proceedings of Machine Learning Research}, pages 252--260,
  Fort Lauderdale, FL, USA, 20--22 Apr 2017{\natexlab{a}}. PMLR.

\bibitem[Bahmani and Romberg(2017{\natexlab{b}})]{Bahmani2017Flexible}
S.~Bahmani and J.~Romberg.
\newblock A flexible convex relaxation for phase retrieval.
\newblock \emph{Electronic Journal of Statistics}, 11\penalty0 (2):\penalty0
  5254--5281, 2017{\natexlab{b}}.

\bibitem[Bahmani and Romberg(2018)]{Bahmani2017Anchored}
S.~Bahmani and J.~Romberg.
\newblock Solving equations of random convex functions via anchored regression.
\newblock \emph{J. Found. Comp. Math.}, 2018.
\newblock In press; preprint
  \href{https://arxiv.org/abs/1702.05327}{\texttt{arXiv:1702.05327 [cs.LG]}}.

\bibitem[Bousquet et~al.(2002)Bousquet, Koltchinskii, and
  Panchenko]{Bousquet2002Local}
O.~Bousquet, V.~Koltchinskii, and D.~Panchenko.
\newblock Some local measures of complexity of convex hulls and generalization
  bounds.
\newblock In J.~Kivinen and R.~H. Sloan, editors, \emph{Computational Learning
  Theory}, pages 59--73, Berlin, Heidelberg, 2002. Springer Berlin Heidelberg.

\bibitem[Cand\`{e}s et~al.(2015)Cand\`{e}s, Li, and
  Soltanolkotabi]{Candes2014Phase}
E.~J. Cand\`{e}s, X.~Li, and M.~Soltanolkotabi.
\newblock Phase retrieval via {Wirtinger} flow: Theory and algorithms.
\newblock \emph{Information Theory, {IEEE} Transactions on}, 61\penalty0
  (4):\penalty0 1985--2007, Apr. 2015.

\bibitem[Catoni(2007)]{Catoni2007PAC-Bayesian}
O.~Catoni.
\newblock \emph{{PAC-Bayesian} supervised classification: the thermodynamics of
  statistical learning}, volume~56 of \emph{Lecture Notes--Monograph Series}.
\newblock Institute of Mathematical Statistics, Beachwood, OH, USA, 2007.

\bibitem[Catoni and Giulini(2017)]{Catoni2017Dimension-free}
O.~Catoni and I.~Giulini.
\newblock Dimension-free {PAC}-{Bayesian} bounds for matrices, vectors, and
  linear least squares regression.
\newblock \emph{arXiv preprint}, Dec. 2017.
\newblock \href{http://arxiv.org/abs/1712.02747}{\texttt{arXiv:1712.02747
  [math, stat]}}.

\bibitem[Chen et~al.(2012)Chen, Gittens, and Tropp]{Chen2012Masked}
R.~Y. Chen, A.~Gittens, and J.~A. Tropp.
\newblock The masked sample covariance estimator: {An} analysis using matrix
  concentration inequalities.
\newblock \emph{Information and Inference: A Journal of the IMA}, 1\penalty0
  (1):\penalty0 2--20, 2012.

\bibitem[Chen and Cand\'{e}s(2015)]{Chen2015Solving}
Y.~Chen and E.~Cand\'{e}s.
\newblock Solving random quadratic systems of equations is nearly as easy as
  solving linear systems.
\newblock In \emph{Advances in Neural Information Processing Systems 28}, pages
  739--747. Curran Associates, Inc., 2015.

\bibitem[Germain et~al.(2009)Germain, Lacasse, Laviolette, and
  Marchand]{Germain2009PAC-Beyesian}
P.~Germain, A.~Lacasse, F.~Laviolette, and M.~Marchand.
\newblock {PAC-Bayesian} learning of linear classifiers.
\newblock In \emph{Proceedings of the 26th Annual International Conference on
  Machine Learning}, ICML '09, pages 353--360, New York, NY, USA, 2009. ACM.

\bibitem[Gin\'{e} and Koltchinskii(2006)]{Gine2006Concentration}
E.~Gin\'{e} and V.~Koltchinskii.
\newblock Concentration inequalities and asymptotic results for ratio type
  empirical processes.
\newblock \emph{Annals of Probability}, 34\penalty0 (3):\penalty0 1143--1216,
  May 2006.

\bibitem[Gin{\'e} et~al.(2003)Gin{\'e}, Koltchinskii, and
  Wellner]{Gine2003Ratio}
E.~Gin{\'e}, V.~Koltchinskii, and J.~A. Wellner.
\newblock Ratio limit theorems for empirical processes.
\newblock In E.~Gin{\'e}, C.~Houdr{\'e}, and D.~Nualart, editors,
  \emph{Stochastic Inequalities and Applications}, pages 249--278, Basel, 2003.
  Birkh{\"a}user Basel.

\bibitem[Goldstein and Studer(2017)]{Goldstein2017Convex}
T.~Goldstein and C.~Studer.
\newblock Convex phase retrieval without lifting via {P}hase{M}ax.
\newblock In D.~Precup and Y.~W. Teh, editors, \emph{Proceedings of the 34th
  International Conference on Machine Learning}, volume~70 of \emph{Proceedings
  of Machine Learning Research}, pages 1273--1281, International Convention
  Centre, Sydney, Australia, 06--11 Aug 2017. PMLR.

\bibitem[Goldstein and Studer(2018)]{Goldstein2018PhaseMax}
T.~Goldstein and C.~Studer.
\newblock Phasemax: Convex phase retrieval via basis pursuit.
\newblock \emph{{IEEE} Transactions on Information Theory}, 64\penalty0
  (4):\penalty0 2675--2689, April 2018.

\bibitem[Grant and Boyd(2014)]{CVX}
M.~Grant and S.~Boyd.
\newblock {CVX}: Matlab software for disciplined convex programming, version
  2.1, Mar. 2014.
\newblock URL \url{http://cvxr.com/cvx}.

\bibitem[{Gurobi Optimization, Inc.}(2016)]{Gurobi}
{Gurobi Optimization, Inc.}
\newblock Gurobi optimizer reference manual, 2016.
\newblock URL \url{http://www.gurobi.com}.

\bibitem[Hartman(1959)]{Hartman1959Functions}
P.~Hartman.
\newblock On functions representable as a difference of convex functions.
\newblock \emph{Pacific J. Math.}, 9\penalty0 (3):\penalty0 707--713, 1959.

\bibitem[Junge and Zeng(2013)]{Junge2013Noncommutative}
M.~Junge and Q.~Zeng.
\newblock Noncommutative {Bennett} and {Rosenthal} inequalities.
\newblock \emph{Annals of Probability}, 41\penalty0 (6):\penalty0 4287--4316,
  Nov. 2013.

\bibitem[Koltchinskii(2001)]{Koltchinskii2001Rademacher}
V.~Koltchinskii.
\newblock Rademacher penalties and structural risk minimization.
\newblock \emph{{IEEE} Transactions on Information Theory}, 47\penalty0
  (5):\penalty0 1902--1914, 2001.

\bibitem[Koltchinskii(2011)]{Koltchinskii2011Oracle}
V.~Koltchinskii.
\newblock \emph{Oracle Inequalities in Empirical Risk Minimization and Sparse
  Recovery Problems}.
\newblock Lecture Notes in Mathematics: \'{E}cole d'\'{E}t\'{e} de
  Probabilit\'{e}s de Saint-Flour XXXVIII-2008. Springer-Verlag Berlin
  Heidelberg, 2011.

\bibitem[Koltchinskii and Mendelson(2015)]{Koltchinskii2015Bounding}
V.~Koltchinskii and S.~Mendelson.
\newblock Bounding the smallest singular value of a random matrix without
  concentration.
\newblock \emph{International Mathematics Research Notices}, 2015\penalty0
  (23):\penalty0 12991--13008, 2015.

\bibitem[Koltchinskii and Panchenko(2000)]{Koltchinskii2000RademacherProcesses}
V.~Koltchinskii and D.~Panchenko.
\newblock Rademacher processes and bounding the risk of function learning.
\newblock In E.~Gin{\'e}, D.~M. Mason, and J.~A. Wellner, editors, \emph{High
  Dimensional Probability II}, pages 443--457, Boston, MA, 2000. Birkh{\"a}user
  Boston.

\bibitem[Langford and Shawe-Taylor(2003)]{Langford2003PAC-Bayesian}
J.~Langford and J.~Shawe-Taylor.
\newblock {PAC-Bayes} \& margins.
\newblock In \emph{Advances in Neural Information Processing Systems}, pages
  439--446, 2003.

\bibitem[Ledoux and Talagrand(2013)]{Ledoux2013Probability}
M.~Ledoux and M.~Talagrand.
\newblock \emph{Probability in {Banach} Spaces: {I}soperimetry and processes}.
\newblock Springer Science \& Business Media, 2013.

\bibitem[Li et~al.(2018)Li, Ling, Strohmer, and Wei]{Li2018Rapid}
X.~Li, S.~Ling, T.~Strohmer, and K.~Wei.
\newblock Rapid, robust, and reliable blind deconvolution via nonconvex
  optimization.
\newblock \emph{Applied and Computational Harmonic Analysis}, 2018.
\newblock \doi{10.1016/j.acha.2018.01.001}.
\newblock in press.

\bibitem[Ling and Strohmer(2018)]{Ling2018Regularized}
S.~Ling and T.~Strohmer.
\newblock Regularized gradient descent: {A} non-convex recipe for fast joint
  blind deconvolution and demixing.
\newblock \emph{Information and Inference: {A} Journal of the {IMA}}, 2018.
\newblock \doi{10.1093/imaiai/iax022}.
\newblock in press.

\bibitem[Luo et~al.(2018)Luo, Alghamdi, and Lu]{Luo2018Optimal}
W.~Luo, W.~Alghamdi, and Y.~M. Lu.
\newblock Optimal spectral initialization for signal recovery with applications
  to phase retrieval.
\newblock Preprint
  \href{https://arxiv.org/abs/1811.04420}{\texttt{arXiv:1811.04420 [cs.IT]}},
  2018.

\bibitem[Ma et~al.(2017)Ma, Wang, Chi, and Chen]{Ma2017Implicit}
C.~Ma, K.~Wang, Y.~Chi, and Y.~Chen.
\newblock Implicit regularization in nonconvex statistical estimation:
  {Gradient} descent converges linearly for phase retrieval, matrix completion
  and blind deconvolution.
\newblock Preprint
  \href{https://arxiv.org/abs/1711.10467}{\texttt{arXiv:1711.10467 [cs.LG]}},
  2017.

\bibitem[Mackey et~al.(2014)Mackey, Jordan, Chen, Farrell, and
  Tropp]{Mackey2014Matrix}
L.~Mackey, M.~I. Jordan, R.~Y. Chen, B.~Farrell, and J.~A. Tropp.
\newblock Matrix concentration inequalities via the method of exchangeable
  pairs.
\newblock \emph{Annals of Probability}, 42\penalty0 (3):\penalty0 906--945, May
  2014.

\bibitem[McAllester and Akinbiyi(2013)]{McAllester2013}
D.~McAllester and T.~Akinbiyi.
\newblock \emph{{PAC-Bayesian} Theory}, pages 95--103.
\newblock Springer Berlin Heidelberg, Berlin, Heidelberg, 2013.

\bibitem[McAllester(1999)]{McAllester1999PAC-Bayesian}
D.~A. McAllester.
\newblock Some {PAC-Bayesian} theorems.
\newblock \emph{Machine Learning}, 37\penalty0 (3):\penalty0 355--363, Dec
  1999.

\bibitem[Mendelson(2014)]{Mendelson2014Learning}
S.~Mendelson.
\newblock Learning without concentration.
\newblock In \emph{Proceedings of the 27th Conference on Learning Theory
  ({COLT})}, volume~35 of \emph{{JMLR W\&CP}}, pages 25--39, 2014.

\bibitem[Mendelson(2015)]{Mendelson2015Learning}
S.~Mendelson.
\newblock Learning without concentration.
\newblock \emph{Journal of the {ACM}}, 62\penalty0 (3):\penalty0 21:1--21:25,
  June 2015.
\newblock ISSN 0004-5411.

\bibitem[Mondelli and Montanari(2018)]{Mondelli2018Fundamental}
M.~Mondelli and A.~Montanari.
\newblock Fundamental limits of weak recovery with applications to phase
  retrieval.
\newblock In \emph{Proceedings of the 31st Conference On Learning Theory
  ({COLT})}, volume~75 of \emph{Proceedings of Machine Learning Research},
  pages 1445--1450. PMLR, 2018.

\bibitem[Netrapalli et~al.(2013)Netrapalli, Jain, and
  Sanghavi]{Netrapalli2013Phase}
P.~Netrapalli, P.~Jain, and S.~Sanghavi.
\newblock Phase retrieval using alternating minimization.
\newblock In \emph{Advances in Neural Information Processing Systems 26}, pages
  2796--2804. Curran Associates, Inc., 2013.

\bibitem[Oliveira(2016)]{Oliveira2016LowerTail}
R.~I. Oliveira.
\newblock The lower tail of random quadratic forms with applications to
  ordinary least squares.
\newblock \emph{Probability Theory and Related Fields}, 166\penalty0
  (3):\penalty0 1175--1194, Dec 2016.

\bibitem[Plan and Vershynin(2013)]{Plan2013Robust}
Y.~Plan and R.~Vershynin.
\newblock Robust 1-bit compressed sensing and sparse logistic regression: {A}
  convex programming approach.
\newblock \emph{IEEE Transactions on Information Theory}, 59\penalty0
  (1):\penalty0 482--494, Jan 2013.

\bibitem[Plan and Vershynin(2016)]{Plan2016Generalized}
Y.~Plan and R.~Vershynin.
\newblock The generalized lasso with non-linear observations.
\newblock \emph{IEEE Transactions on Information Theory}, 62\penalty0
  (3):\penalty0 1528--1537, March 2016.

\bibitem[{van Der Vaart} and Wellner(1996)]{vanderVaart1996Weak}
A.~W. {van Der Vaart} and J.~A. Wellner.
\newblock \emph{Weak Convergence and Empirical Processes}.
\newblock {Springer} {Series} in {Statistics}. Springer, 1996.

\bibitem[Vapnik(1998)]{Vapnik1998Statistical}
V.~N. Vapnik.
\newblock \emph{Statistical learning theory}.
\newblock Wiley, 1998.

\bibitem[Vapnik and Chervonenkis(1971)]{Vapnik1971Uniform}
V.~N. Vapnik and A.~Y. Chervonenkis.
\newblock On the uniform convergence of relative frequencies of events to their
  probabilities.
\newblock \emph{Theory of Probability \& Its Applications}, 16\penalty0
  (2):\penalty0 264--280, 1971.

\bibitem[Vershynin(2012)]{Vershynin2012Introduction}
R.~Vershynin.
\newblock Introduction to the non-asymptotic analysis of random matrices.
\newblock In G.~Kutyniok and Y.~Eldar, editors, \emph{Compressed Sensing,
  Theory and Applications}, pages 210--268. Cambridge University Press, 2012.

\bibitem[Yu et~al.(2015)Yu, Wang, and Samworth]{Yu2015Useful}
Y.~Yu, T.~Wang, and R.~J. Samworth.
\newblock A useful variant of the {Davis}-{Kahan} theorem for statisticians.
\newblock \emph{Biometrika}, 102\penalty0 (2):\penalty0 315--323, 2015.

\end{thebibliography}

\end{document}